\documentclass{ecai}
\usepackage{times}
\usepackage{graphicx}
\usepackage{latexsym}
\usepackage{txfonts}

\RequirePackage{booktabs} 
\usepackage{xfrac} %% pour les sfrac (fractions avec barre penchée)
\RequirePackage{amsthm,amsmath,amsfonts,amssymb}

\newtheorem{definition}{Definition}
\newtheorem{example}{Example}

\newtheorem{proposition}{Proposition}
\newtheorem{lemma}{Lemma}

\usepackage{bm}

\newcommand{\ie}{\emph{i.e.~}}

\newcommand{\eg}{\emph{e.g.~}}
\newcommand{\al}{\emph{al.~}}
\newcommand{\espace}{\vspace{7pt}}
\newcommand{\seq}[1]{\bm{#1}}

\newcommand{\ignore}[1]{}

\usepackage{xcolor}

\usepackage{relsize}

\usepackage{pifont}

\usepackage{calc}

\usepackage{mathtools}
\newcommand\myeq{\stackrel{\mathclap{\normalfont}}{:=}}
\newcommand*{\dom}{\eqslantgtr}
\newcommand{\nodom}{-}

\newcommand*\nspincl{\mathrel{\ooalign{$\lhd$\cr\hidewidth\hbox{$\cdot\mkern 2.5mu$}\cr}}}

\newcommand*\nspinclplus{\mathrel{\ooalign{$\lhd$\vspace{-1.2pt}\cr\hidewidth\hbox{\tiny$+\mkern 4mu$}\cr}}}

\newcommand*\gennonincl{\mathlarger{*}}
\newcommand*\partialnonincl{G}
\newcommand*\totalnonincl{D}

\newcommand*\gennoninclrel{\not\subseteq_{{\mathlarger{*}}}}
\newcommand*\partialnoninclrel{\not\subseteq_{G}}
\newcommand*\totalnoninclrel{\not\subseteq_{D}}

\newcommand*\strictemb{
		 \mbox{\parbox{0.8ex}{
		 \begin{tikzpicture}[scale=0.4]
 		 \draw (0,0) circle (1ex);\fill (1ex,0) arc (0:360:1ex) -- (0,0) -- cycle;
		 \end{tikzpicture}
		 }}
}
\newcommand*\softemb{
		 \mbox{\parbox{0.8ex}{
		 \begin{tikzpicture}[scale=0.4]
 		 \draw (0,0) circle (1ex);
		 \end{tikzpicture}
		 }}
}
\newcommand*\genemb{
		 \mbox{\parbox{0.8ex}{
		 \begin{tikzpicture}[scale=0.4]
 		 \draw (0,0) circle (1ex);\fill (0,1ex) arc (90:270:1ex) -- (0,0) -- cycle;
		 \end{tikzpicture}
		 }}
}

\newcommand\stronglycontains{\sqsubseteq}
\newcommand\weaklycontains{\preceq}
\newcommand\gencontains{\strictif}

\usepackage{tikz}

\usepackage{hyperref}
\usepackage{cleveref}
\crefformat{footnote}{#2\footnotemark[#1]#3}

\ecaisubmission   % inserts page numbers. Use only for submission of paper.
\begin{document}

\title{Semantics of negative sequential patterns}

\author{Philippe Besnard \institute{CNRS\,/\,IRIT, France, email: besnard@irit.fr} \and Guyet Thomas
\institute{Institut~Agro~/~IRISA UMR6074, France, email: thomas.guyet@irisa.fr} }

\maketitle
\bibliographystyle{ecai}

\begin{abstract}
In the field of pattern mining, a negative sequential pattern is specified by means of a sequence consisting of events to occur and of other events, called negative events, to be absent. For instance, containment of the pattern $\langle a\ \neg b\ c\rangle$ arises with an occurrence of $a$ and a subsequent occurrence of $c$ but no occurrence of $b$ in between.

This article is to shed light on the ambiguity of such a seemingly intuitive notation and we identify eight possible semantics for the containment relation between a pattern and a sequence. These semantics are illustrated and formally studied, in particular we propose dominance and equivalence relations between them. Also we prove that support is anti-monotonic for some of these semantics. Some of the results are discussed with the aim of developing algorithms to extract efficiently frequent negative patterns.
\end{abstract}

%%%%%%%%%%%%%%%%%%%%%%%%%%%%%%%%%%%%%%%%%%%%%%%%%%%%%%%%%%%%%%%%%%%%%%%%%%%%%%%%%%%%%%%%%%%%%%%%%%%%%%%%%
%%%%%%%%%%%%%%%%%%%%%%%%%%%%%%%%%%%%%%%%%%%%%%%%%%%%%%%%%%%%%%%%%%%%%%%%%%%%%%%%%%%%%%%%%%%%%%%%%%%%%%%%%
\section{Introduction}
In many application domains such as predictive maintenance or marketing, decision makers are interested in discovering specific events that trigger or are correlated to undesirable events. Sequential pattern mining \cite{Mooney:2013} is a technique that extracts such hidden rules from logs. 

Often, the presence but also the absence of a specific action or event partly explains the occurrence of an undesirable situation \cite{Cao2015}. 
For example in predictive maintenance, if some maintenance operations have not been performed, \eg damaged parts have not been replaced, then a fault 
is likely to occur in a short delay whereas if these operations were performed in time the fault would not occur. In marketing, if a marketplace customer has not received special offers or coupons for a long time then s/he has a high probability of churning whereas if s/he were provided with such offers s/he should remain loyal to her/his marketplace.
Mining specific events to discover a context under which they occur, or do not occur, may provide interesting information.
It is called \emph{actionable} information as it serves to determine what action should be performed to avoid the undesirable situation, \ie fault in monitored systems, churn in marketing, \ldots

\begin{figure}
\caption{On the left, a synthetic dataset of six sequences. On the right, a set of rules with their support and accuracy in the dataset.}
\label{tab:intro}
\centering
\renewcommand{\arraystretch}{1.2}
\begin{tabular}{@{}ll@{}}
\toprule 
$\seq{s_1}$ & $\langle a\ c\ e\rangle$ \\%[-0.25ex]
$\seq{s_2}$ & $\langle a\ b\ c\ e\rangle$ \\%[-0.25ex]
$\seq{s_3}$ & $\langle a\ b\ c\ e\rangle$ \\%[-0.25ex]
$\seq{s_4}$ & $\langle a\ c\ d\rangle$ \\%[-0.25ex]
$\seq{s_5}$ & $\langle a\ c\ d\rangle$ \\%[-0.25ex]
$\seq{s_6}$ & $\langle a\ c\ d\rangle$ \\%[-0.25ex]
\bottomrule
\end{tabular}\hspace{30pt}
\begin{tabular}{@{}lcc@{}}
\toprule 
Rule & support & accuracy\\
\midrule
$\langle a\ c\rangle \implies e$ & $3$ & $\sfrac{1}{2}$ \\
$\langle a\ c\rangle \implies d$ & $3$ & $\sfrac{1}{2}$ \\
$\langle a\ \neg b\ c\rangle \implies e$ & $1$ & $\sfrac{1}{4}$ \\
$\langle a\ \neg b\ c\rangle \implies d$ & $3$ & $\sfrac{3}{4}$  \\
\bottomrule
\end{tabular}
\end{figure}

Standard sequential pattern mining algorithms \cite{Mooney:2013} extract sequential patterns that frequently occur in the logs. 
A sequential pattern is a sequence of events. For example, the sequential pattern $\langle a\ c\ d\rangle$ is read as ``$a$ occurs and then $c$ occurs and finally $d$ occurs''. In practice, a pattern is frequent if its number of occurrences exceeds a user-defined threshold. 
If $\langle a\ c\ d\rangle$ occurs in most cases where $\langle a\ c\rangle$ occurs, then the sequential rule $\langle a\ c\rangle \implies d $ (read as ``if $a$ occurs and $c$ occurs later, then $d$ occurs afterwards'') is useful to predict occurrences of $d$. 
The premise of a rule specifies what actually occurred frequently, but does not inform about what did not happen in these examples. 
Negative sequential patterns are sequential patterns %with the specification of 
that also specify non-occurring events. 
Intuitively, the syntax of a simple negative sequential pattern is as follows: $\langle a\ \neg b\ c\rangle$. This pattern is read as ``$a$ occurs and then $c$ occurs, but $b$ does not occur in between''. A negative sequential pattern can also be the premise of a rule.

We illustrate the interest of negative sequential patterns via the dataset of sequences in Figure \ref{tab:intro}. The rightmost table gives the support (number of sequences containing both premise and conclusion) and the accuracy (ratio of the support with the support of the premise only) of some %sequential 
rules. The sequential patterns $\langle a\ c\ d\rangle$ and $\langle a\ c\ e\rangle$ occur thrice each. Rules $\langle a\ c\rangle \Rightarrow d$
and $\langle a\ c\rangle \Rightarrow e$ obtained from positive sequential patterns have low accuracy, they are not really interesting. 
Let $\seq{s_p}=\langle a\ b\ c\ ?\rangle$ be a new sequence with an event to predict. The two rules above predict $e$ or $d$ with the same likelihood. 

Modeling the absence of event $b$ in patterns appears to be meaningful to describe the dataset. Indeed, rule $\langle a\ \neg b\ c\rangle \Rightarrow d$ occurs in half of the sequences and has an accuracy of $\sfrac{3}{4}$ 
(whereas the accuracy of the rule without $\neg b$ is only $\sfrac{1}{2}$).
When it comes to predicting occurrences of $d$, the absence of $b$ is meaningful. These new rules predict event $e$ with a likelihood of $\sfrac{3}{4}$ for $\seq{s_p}$. 
In a medical context, $a$, $b$ and $c$ may be drug administration while $d$ and $e$ some medical events, respectively, patient declared cured and patient suffering complications. The situation that is illustrated by our synthetic dataset is the case of adverse drugs reaction. Being exposed to drug $b$ while being treated by drugs $a$ and $c$ leads to complications. Mining positive patterns in a medical database would miss such adverse drug reaction.

Mining frequent negative sequential patterns is of utmost interest to discover actionable rules taking into account absent events. 
In \cite{imielinski1996database}, pattern mining is viewed as the computation of a theory $Th(\mathcal{L}, \mathcal{D}, \mathcal{C} )$ $=$ $\left\{\psi \in \mathcal{L} \ |\  \mathcal{C}(\psi, \mathcal{D})\right\}$. 
Given a pattern language $\mathcal{L}$, some constraints $\mathcal{C}$ and a database $\mathcal{D}$, a pattern mining algorithm enumerates the elements of the language that fulfill the constraints within the data.
In the case of frequent pattern mining, $\mathcal{C}$ is the minimal support constraint.
The success of pattern mining techniques comes from an anti-monotonicity property of some support measures~\cite{Agrawal:1994}. 
Intuitively, if a pattern $\seq{p}$ is not frequent, no pattern ``larger'' than $\seq{p}$ is frequent. Pattern mining algorithms prune the search space whenever an unfrequent pattern is found. 
The ``is larger than'' relation induces a partial order on the set of patterns, $\mathcal{L}$.
For a support measure that is anti-monotonic on this structure, the frequent pattern mining trick can be used to efficiently prune the search space. 
Ideally, this structure is a lattice, in which case the above strategy is complete and correct.

As to frequent negative sequential pattern mining, $\mathcal{L}$ is the set of negative sequential patterns, $\mathcal{D}$ is a dataset of sequences and $\mathcal{C}$ is the constraint of minimal frequency.
Few approaches \cite{cao2016nsp,ChenCK03,Gong2017eNSPFI,Guyet2020,hsueh:2008:PNSP,xu2017msnsp,zheng:2009:negative} proposed algorithms to extract such patterns and none of them proposed an algorithm based on an anti-monotonic support measure. 
The questions we address in this article are:
\begin{enumerate}
\item what is a proper support measure for negative sequential patterns?
\item is there a support measure enjoying anti-monotonicity?
\end{enumerate}

The support measure is strongly related to the \textit{containment relation} that determines whether a pattern occurs in a sequence or not. 
In the case of negative sequential patterns, the apparently intuitive notion of absent event appears to be intricate and the negation syntax (the $\neg$ symbol) used in the literature is hiding different semantics.
In logic, it is accepted knowledge that there is more than one kind of negation \cite{wansing2017negation}. For instance, in classical reasoning $\neg p$ means that $p$ is false while in stable reasoning $\neg p$ means that $p$ cannot be proved \cite{cabalar2017stable}.

The objective of this article is not to propose a new pattern mining algorithm for negative sequential patterns but to establish formal results on containment relations that can serve as a basis to design such algorithms. 
The main contributions of our work are as follows:
\begin{itemize}
\item we define eight possible semantics for the containment relation of negative sequential patterns, 
\item we establish dominance and equivalence relations between containment relations,
\item we provide three partial orders for which some containment relations induce anti-monotonic support measures.
\end{itemize}

\section{Negative sequential patterns}\label{sec:negpatterns:definitions}
Throughout this article, $[n]=\{1, \dots, n\}$ denotes the set of the first $n$ positive integers.
Let $\mathcal{I}$ be the set of items (alphabet). 
An \emph{itemset} $A=\{a_1\ a_2\ \cdots\ a_m\}\subseteq \mathcal{I}$ is a 
finite 
set of items. 
The length of $A$, denoted $|A|$, is $m$.
A \emph{sequence} $\seq{s}$ is of the form $\seq{s} = \langle s_1\ s_2\ \cdots\ s_n\rangle$ where each $s_i$ is an itemset.

\begin{definition}[Negative sequential patterns (NSP)]\label{def:negativepattern}
A negative sequential pattern $\seq{p} = \langle p_1\ \neg q_1 \ p_2\  \neg q_2\ \cdots $ $p_{n-1}\ \neg q_{n-1}\ p_n\rangle$ is a finite sequence where $p_i \in 2^{\mathcal{I}}\setminus\{\emptyset\}$ for all $i\in [n]$ and $q_i \in 2^{\mathcal{I}}$ for all $i\in [n-1]$.

The \emph{length} of $\seq{p}$, denoted $|\seq{p}|$ is the number of its non empty itemsets (negative or positive).

$\seq{p}^+=\langle p_1\,\dots\ p_n\rangle$ is called the positive part of the NSP.
\end{definition}

We denote by $\mathcal{N}$ the set of negative sequential patterns.

It can be noticed that Definition \ref{def:negativepattern} introduces syntactic limitations on negative sequential patterns that are commonly encountered in the state of the art \cite{Wang2019}:
\begin{itemize}
\item a pattern can neither start or finish by a negative itemset,
\item a pattern cannot have two successive negative itemsets.
\end{itemize}

\begin{example}[Negative sequential pattern]
This example illustrates the notations introduced in Definition \ref{def:negativepattern}.
Consider $\mathcal{I}=\{a,b,c,d\}$ and $\seq{p}=\langle a\ \neg (bc)\ (ad)\ d\ \neg (ab)\ d\rangle$.
Let $p_1=\{a\}$,  $p_2=\{ad\}$, $p_3=\{d\}$, $p_4=\{d\}$ and $q_1=\{bc\}$, $q_2=\emptyset$, $q_3=\{ab\}$.
The length of $\seq{p}$ is $|\seq{p}|=6$ and $\seq{p}^+=\langle a\ (ad)\ d\ d\rangle$. 
\end{example}

%%%%%%%%%%%%%%%%%%%%%%%%%%%%%%%%%%%%%%%%%%%%%%%%%%%%%%%%%%%%%%%%
\section{Semantics of negative sequential patterns}\label{sec:negpatterns:semantics}
The semantics of negative sequential patterns relies upon \emph{negative containment}: a sequence $\seq{s}$ supports pattern $\seq{p}$ (or $\seq{p}$ matches the sequence $\seq{s}$) iff $\seq{s}$ contains a sub-sequence $\seq{s}'$ such that every positive itemset of $\seq{p}$ is included in some itemset of $\seq{s}'$ in the same order and for any negative itemset $\neg q_i$ of $\seq{p}$, $q_i$ is \emph{not included} in any itemset occurring in the sub-sequence of $\seq{s}'$ located between the occurrence of the positive itemset preceding $\neg q_i$ in $\seq{p}$ and the occurrence of the positive itemset following $\neg q_i$ in $\seq{p}$.

\begin{definition}[Non inclusion]\label{def:IS_notincluded}
We introduce two relations comparing two itemsets $P \in 2^{\mathcal{I}}\setminus\{\emptyset\}$
and $I\in 2^{\mathcal{I}}$:
\begin{itemize}
\item partial non inclusion: $P\partialnoninclrel I \Leftrightarrow \exists e \in P$, $e \notin I$
\item total non inclusion: $P\totalnoninclrel I \Leftrightarrow \forall e \in P, e \notin I$
\end{itemize}
Partial non-inclusion means that $P \setminus I$ is non-empty while total non-inclusion means that $P$ and $I$ are disjoint.
By convention, $\emptyset\totalnoninclrel I$ and $\emptyset\partialnoninclrel I$ for all $I\subseteq\mathcal{I}$.
\end{definition}

In the sequel we will denote the general form of itemset non-inclusion by the symbol $\gennoninclrel$,  meaning either $\partialnoninclrel$ or $\totalnoninclrel$.

Intuitively, partial non-inclusion identifies the itemset $P$ with a disjunction of negative constraints, \ie at least one of the items (of $P$) has to be absent from $I$, and total non-inclusion consider the itemset $P$ as a conjunction of negative constraints: all items (of $P$) have to be absent from $I$.

Choosing one non-inclusion interpretation or the other has consequences on extracted patterns as well as on pattern search. Let us illustrate this with the following dataset of sequences:
{\small
$$\mathcal{D} = \left\{
\begin{array}{l}
\seq{s}_1=\langle (bc)\ f\ a \rangle \\
\seq{s}_2=\langle (bc)\ (cf)\ a \rangle \\
\seq{s}_3=\langle (bc)\ (df)\ a \rangle \\
\seq{s}_4=\langle (bc)\ (ef)\ a \rangle\\
\seq{s}_5=\langle (bc)\ (cdef)\ a \rangle
\end{array}
\right\}.$$}
Table \ref{tab:partial-total} compares the support of patterns under the two semantics of itemset non-inclusion. 
Since the positive part of $\seq{p}_2$ is in $\seq{s}_2$, $\seq{p}_2$ occurs in the sequence iff $(cd)\gennoninclrel (cf)$. As for total non-inclusion, it is false that $(cd)\totalnoninclrel (cf)$ because $c$ occurs in $(cf)$, and thus $\seq{p}_2$ does not occur in $\seq{s}_2$. As for partial non-inclusion, it is true that $(cd)\partialnoninclrel (cf)$, because $d$ does not occur in $(cf)$, and thus $\seq{p}_2$ occurs in $\seq{s}_2$.

\begin{lemma}\label{lemma:non_incl_relation}\footnote{All proofs can be found in the appendix of this article.}
Let $P, I \subseteq \mathcal{I}$ be two itemsets: 
\begin{equation}
P\totalnoninclrel I \implies P\partialnoninclrel I \label{eq:non_incl_relation}
\end{equation}
\end{lemma}

\begin{table}
\small\centering
\caption{Lists of sequences in $\mathcal{D}$ supported by negative patterns $(\seq{p}_i)_{i=1..4}$ under the total and partial non-inclusion relations. Each pattern has the form $\langle b\ \neg q_i\ a\rangle$ where $q_i$ are itemsets such that $q_i \subset q_{i+1}$.}
\label{tab:partial-total}

\renewcommand{\arraystretch}{1.2}
    \begin{tabular}{@{}lcc@{}}
        \toprule
        ~                               & partial & total \\ 
        ~                               & non-inclusion & non-inclusion \\
        ~ & $\partialnoninclrel$ & $\totalnoninclrel$\\
        \midrule
        $\seq{p}_1 = \langle b\ \neg c\ a \rangle$      & $\{\seq{s}_1, \seq{s}_3, \seq{s}_4\}$           & $\{\seq{s}_1, \seq{s}_3, \seq{s}_4\}$ \\ 
        $\seq{p}_2 = \langle b\ \neg (cd)\ a \rangle$   & $\{\seq{s}_1, \seq{s}_2, \seq{s}_3, \seq{s}_4\}$ & $\{\seq{s}_1, \seq{s}_4\}$ \\ 
        $\seq{p}_3 = \langle b\ \neg (cde)\ a \rangle$  & $\{\seq{s}_1, \seq{s}_2, \seq{s}_3, \seq{s}_4\}$   & $\{\seq{s}_1\}$ \\ 
        $\seq{p}_4 = \langle b\ \neg (cdeg)\ a \rangle$ &$\{\seq{s}_1, \seq{s}_2, \seq{s}_3, \seq{s}_4,\seq{s}_5\}$ & $\{\seq{s}_1\}$ \\
        \bottomrule
    \end{tabular}
\end{table}

Now, we formulate the notions of sub-sequence, non-inclusion and absence by means of the concept of embedding.

\begin{definition}[Positive pattern embedding]\label{def:positivepattern_embedding}
Let $\seq{s}=\langle s_1\,\dots\, s_n\rangle$ be a sequence and $\seq{p}=\langle p_1\,\dots\, p_m\rangle$ be a (positive) sequential pattern.
A~tuple
$\seq{e}=(e_i)_{i\in[m]}\in [n]^m$ is an \emph{embedding} of pattern $\seq{p}$ in sequence $\seq{s}$ iff $\forall i\in[m],\; p_i \subseteq s_{e_i}$ and $e_{i}<e_{i+1}$ for all $i\in[m-1]$.
\end{definition}

\begin{definition}[Strict and soft embeddings of negative patterns]\label{def:NSP_embedding}
Let $\seq{s}=\langle s_1\,\dots\, s_n\rangle$ be a sequence and $\seq{p}=\langle p_1\ \neg q_1\ \dots\ \ \neg q_{m-1}\ p_m\rangle$ be a negative sequential pattern.

An increasing\footnote{\label{fn:increasing}By an increasing tuple $\seq{e}$, we mean a tuple such that $e_i < e_{i+1}$ (in particular, repetitions are not allowed).} tuple 
$\seq{e}=(e_i)_{i\in[m]}\in [n]^m$ is a 
$\softemb$-embedding (read: \textbf{soft-embedding})
of pattern $\seq{p}$ in sequence $\seq{s}$ iff:
\begin{itemize}
\item $p_i \subseteq s_{e_i}$ for all $i\in[m]$
\item $q_i \gennoninclrel s_j,\;\forall j\in [e_{i}+1,e_{i+1}-1]$ for all $i\in[m-1]$
\end{itemize}

An increasing\cref{fn:increasing} tuple $\seq{e}=(e_i)_{i\in[m]}\in [n]^m$ is a $\strictemb$-embedding (read: \textbf{strict-embedding}) of pattern $\seq{p}$ in sequence $\seq{s}$ iff:
\begin{itemize}
\item $p_i \subseteq s_{e_i}$ for all $i\in[m]$
\item $q_i \gennoninclrel \bigcup_{j\in [e_{i}+1,e_{i+1}-1]} s_j$ for all $i\in[m-1]$
\end{itemize}
\end{definition}

Intuitively, the constraint of a negative itemset $q_i$ is checked on the sequence's itemsets at positions in interval $[e_{i}+1,e_{i+1}-1]$, \ie between occurrences of the two positive itemsets surrounding the negative itemset in the pattern. 
A soft embedding considers individually each of the sequence's itemsets of $[e_{i}+1,e_{i+1}-1]$ while a strict embedding consider them as a whole. 

\begin{example}[Itemset absence semantics]\label{ex:itemsetsemantic}
Let $\seq{p}=\langle a\ \neg (bc)\ d\rangle$ be a pattern and consider four sequences as follows:
\begin{center}
\centering
\small
\renewcommand{\arraystretch}{1.2}
\begin{tabular}{@{}lcccc@{}}
\toprule
Sequence & $\totalnoninclrel$ & $\totalnoninclrel$ & $\partialnoninclrel$ & $\partialnoninclrel$\\
~ &$\strictemb$ & $\softemb$ & $\strictemb$ & $\softemb$
 \\ \midrule
$\seq{s_1}=\langle a\ c\ b\ e\ d \rangle$ & & & &\ding{51}\\
$\seq{s_2}=\langle a\ (bc)\ e\ d \rangle$ & & & &\\
$\seq{s_3}=\langle a\ b\ e\ d \rangle$ & & &\ding{51} &\ding{51}\\
$\seq{s_4}=\langle a\ e\ d \rangle$ &\ding{51}&\ding{51}&\ding{51} &\ding{51}\\
\bottomrule
\end{tabular}
\end{center}
 The reader can notice that each sequence contains a unique occurrence of $\seq{p}^+=\langle a\ d \rangle$, the positive part of pattern $\seq{p}$. 
Considering soft-embedding and partial non-inclusion ($\gennoninclrel\myeq\partialnoninclrel$), $\seq{p}$ occurs in $\seq{s_1}$, $\seq{s_3}$ and $\seq{s_4}$ but not in $\seq{s_2}$. 
Considering strict-embedding and partial non-inclusion, $\seq{p}$ occurs in $\seq{s_3}$ and  $\seq{s_4}$. Indeed, items $b$ and $c$ occur between occurrences of $a$ and $d$ in $\seq{s_1}$ and $\seq{s_2}$. 
Considering total non-inclusion ($\gennoninclrel\myeq\totalnoninclrel$) and either type of embeddings, the absence of an itemset is satisfied if any of its items is absent. Hence, $\seq{p}$ occurs only in $\seq{s_4}$. 
\end{example}

\begin{lemma}\label{lemma:bulletimpliescirc}
If $\seq{e}$ is a $\strictemb$-embedding, then $\seq{e}$ is a $\softemb$-embedding, 
regardless of whether $\gennoninclrel$ is $\partialnoninclrel$ or $\totalnoninclrel$.
\end{lemma}

\begin{lemma}\label{prop:sqsubset_eqembeddings}
In the case that $\gennoninclrel$ is $\totalnoninclrel$,
$\seq{e}$ is a $\softemb$-embedding iff $\seq{e}$ is a $\strictemb$-embedding
\end{lemma}

\begin{lemma}\label{lemma:soft_implies_strict}
Let $p=\langle p_1\ \neg q_1\ \dots\ \neg q_{n-1}\ p_n \rangle \in \mathcal{N}$ such that $|q_i|\leq 1$ for all $i\in[n-1]$, then
$\seq{e}$ is a $\softemb$-embedding iff $\seq{e}$ is a $\strictemb$-embedding.
\end{lemma}

Lemma \ref{lemma:soft_implies_strict} shows that in the simple case of patterns 
where all negative itemsets are singleton sets, the notions of strict and soft embeddings coincide.

\begin{lemma}\label{lemma:pos_embedding}
Let $\seq{p} \in \mathcal{N}$, if $\seq{e}$ is an embedding of $\seq{p}$ in some sequence $\seq{s}$, then $\seq{e}$ is an embedding of $\seq{p}^+$ in $\seq{s}$.
\end{lemma}

Another point that determines the semantics of negative containment concerns the multiple occurrences of some pattern in a sequence: should at least one or should all occurrences of the pattern positive part in the sequence satisfy the non-inclusion constraints? 

\begin{definition}[Negative pattern occurrence] \label{def:neg_occurrence}
Let $\seq{s}$ be a sequence and $\seq{p}$ be a negative sequential pattern with $\seq{p}^+$ the positive part of $\seq{p}$. \linebreak[4]
For $\gennoninclrel\in\{\totalnoninclrel,\partialnoninclrel\}$ and $\genemb\in\{\softemb,\strictemb\}$,
\begin{itemize}
\item $\seq{p} \weaklycontains^{\gennonincl}_{\genemb} \seq{s}$ denotes that pattern $\seq{p}$ occurs in sequence $\seq{s}$ iff there exists at least one $\genemb$-embedding of $\seq{p}$ in $\seq{s}$ considering the $\gennoninclrel$ non-inclusion.
\item $\seq{p} \stronglycontains^{\gennonincl}_{\genemb} \seq{s}$ denotes that pattern $\seq{p}$ occurs in sequence $\seq{s}$ iff for each embedding $\seq{e}$ of $\seq{p}^+$ in $\seq{s}$, $\seq{e}$ is also a $\genemb$-embedding of $\seq{p}$ in $\seq{s}$ considering the $\gennoninclrel$ non-inclusion, and there exists at least one embedding $\seq{e}$ of $\seq{p}^+$.
\end{itemize}
\end{definition}

Definition \ref{def:neg_occurrence} permits to capture two semantics for negative sequential patterns depending on the occurrences of the positive part: $\seq{p} \stronglycontains^{\gennonincl}_{\genemb} \seq{s}$ states that a negative pattern $\seq{p}$ occurs in a sequence $\seq{s}$ iff there exists at least one occurrence of the positive part of pattern $\seq{p}$ in sequence $\seq{s}$ and \textbf{every} such occurrence satisfies the negative constraints; $\seq{p} \weaklycontains^{\gennonincl}_{\genemb} \seq{s}$ states that $\seq{p}$ occurs in a sequence $\seq{s}$ iff there exists at least one occurrence of the positive part of pattern $\seq{p}$ in sequence $\seq{s}$ and \textbf{at least one} of these occurrences satisfies the negative constraints.

\begin{example}[Strong vs weak occurrence semantics]
Let $\seq{p}=\langle a\ b\ \neg c\  d \rangle$ be a pattern,  $\seq{s_1}=\langle a\ b\ e\ d \rangle$ and $\seq{s_2}=\langle a\ b\ c\ a\ d\ e\ b\ d \rangle$ be two sequences. 
Thus, $\seq{p}^+=\langle a\ b\ d \rangle$ occurs once in $\seq{s_1}$ hence there is no difference for occurrences of $\seq{p}$ in $\seq{s_1}$ under the two semantics.
However, $\seq{p}^+$ occurs four times in $\seq{s_2}$ through embeddings $(1,2,5)$, $(1,2,8)$, $(1,7,8)$ and $(4,7,8)$. The first two occurrences do not satisfy the negative constraint ($\neg c$) but the last two occurrences do.
Under the weak occurrence semantics, pattern $\seq{p}$ occurs in sequence $\seq{s_2}$ whereas it fails to do so under the strong occurrence semantics.
\end{example}

\begin{lemma}
\label{lemma:strictocc_implies_softocc}
Let $\seq{p}$ be an NSP and $\seq{s}$ a sequence.
For $\genemb\in\{\softemb,\strictemb\}$ and $\gennoninclrel\in\{\totalnoninclrel,\partialnoninclrel\}$,
\begin{equation}
\seq{p}\stronglycontains_{\genemb}^{\gennonincl}\seq{s} \implies \seq{p}\weaklycontains_{\genemb}^{\gennonincl}\seq{s}
\end{equation}
\end{lemma}

\begin{lemma}\label{lemma:notstrict_implies_notsoft}
Let $\seq{p}$ be an NSP and $\seq{s}$ a sequence.
For $\genemb\in\{\softemb,\strictemb\}$,
\begin{eqnarray*}
\seq{p}\weaklycontains_{\genemb}^{\totalnonincl}\seq{s} \implies \seq{p}\weaklycontains_{\genemb}^{\partialnonincl}\seq{s}\\
\seq{p}\stronglycontains_{\genemb}^{\totalnonincl}\seq{s} \implies \seq{p}\stronglycontains_{\genemb}^{\partialnonincl}\seq{s}
\end{eqnarray*}
\end{lemma}

In this section, we have exhibited several semantics that can be associated to negative patterns. This leads to eight different types of pattern occurrences. 
We take $\Theta$ to denote the set of containment relations:
\[\Theta = \left\{
\weaklycontains^{\totalnonincl}_{\softemb}, \weaklycontains^{\totalnonincl}_{\strictemb},
\weaklycontains^{\partialnonincl}_{\softemb}, \weaklycontains^{\partialnonincl}_{\strictemb},
\stronglycontains^{\totalnonincl}_{\softemb}, \stronglycontains^{\totalnonincl}_{\strictemb},
\stronglycontains^{\partialnonincl}_{\softemb}, \stronglycontains^{\partialnonincl}_{\strictemb} 
\right\}
\]

These containment relations allow to disambiguate the semantics of negative pattern containment encountered in the literature. 
Next, Section~\ref{sec:dominance} investigates possible equivalent containment relations in $\Theta$.

\section{Dominance and equivalence between containment relations}\label{sec:dominance}

\begin{definition}[Dominance]\label{def:dominance}
For $\theta,\theta'\in\Theta$, 
$\theta$ \emph{dominates} $\theta'$, denoted $\theta\dom \theta'$, iff $\seq{p} \theta \seq{s} \implies \seq{p} \theta' \seq{s}$ for all $\seq{p}\in \mathcal{N}$ and all sequence $\seq{s}$.

We denote by $\theta\not\dom \theta'$ iff $\theta\dom \theta'$ is false, that is, there is some couple $(\seq{p},\seq{s})$ such that $\seq{p}\theta\seq{s}$ but not $\seq{p}\theta'\seq{s}$.
\end{definition}

The idea behind dominance between two containment relations $\theta$ and $\theta'$ is related to the sequences in which a pattern occurs. 
By definition, if $\theta\dom \theta'$  then a pattern $\seq{p}\in \mathcal{N}$ occurs in a sequence $\seq{s}$ according to the $\theta'$ containment relation whenever $\seq{p}$ occurs in $\seq{s}$ according to the $\theta$ containment relation. 
In the context of pattern mining, this is useful to design algorithms exploiting properties of a dominating containment relation in order to extract efficiently the patterns according to dominated containment relations.

\begin{lemma}\label{lemma:dom_preorder}
The dominance relation $\dom$ is a pre-order.
\end{lemma}

\begin{definition}[Equivalent containment relations]
For $\theta,\theta'\in\Theta$, $\theta$ is equivalent to $\theta'$, denoted $\theta \sim\theta'$ iff $\theta\dom\theta'$ and $\theta'\dom\theta$.
\end{definition}

\begin{lemma}\label{lemma:equivalent_relation}
$\sim$ is an equivalence relation on $\Theta$.
\end{lemma}

Two equivalent containment relations have equivalent semantics, in the following sense: the sets of sequences in which a given pattern occurs are the same and, reciprocally, the sets of negative patterns that occur in a sequence are the same when considering these two containment relations. 

We now study the dominance relations that hold between the elements of $\Theta$.

\begin{proposition}
\label{prop:dominances}
The following dominance statements between containment relations hold:
\begin{gather}
\stronglycontains^{\gennonincl}_{\genemb} \; \dom \; \weaklycontains^{\gennonincl}_{\genemb} \label{eq:dominance3}\\
\weaklycontains^{\gennonincl}_{\strictemb} \; \dom \; \weaklycontains^{\gennonincl}_{\softemb} \text{ and } \stronglycontains^{\gennonincl}_{\strictemb} \; \dom \; \stronglycontains^{\gennonincl}_{\softemb}  \label{eq:dominance1}\\
\weaklycontains^{\totalnonincl}_{\softemb} \; \dom \; \weaklycontains^{\totalnonincl}_{\strictemb} \text{ and } \stronglycontains^{\totalnonincl}_{\softemb} \; \dom \; \stronglycontains^{\totalnonincl}_{\strictemb} \label{eq:dominance2}\\
\weaklycontains^{\totalnonincl}_{\genemb} \; \dom \; \weaklycontains^{\partialnonincl}_{\genemb} \text{ and } \stronglycontains^{\totalnonincl}_{\genemb} \; \dom \; \stronglycontains^{\partialnonincl}_{\genemb} \label{eq:dominance4}
\end{gather}
and the following non-dominance statements hold:\\[-1ex]
{\begin{gather}
\weaklycontains^{\partialnonincl}_{\genemb} \; \not\dom \; \weaklycontains^{\totalnonincl}_{\genemb} \text{ and } \stronglycontains^{\partialnonincl}_{\genemb} \; \not\dom \; \stronglycontains^{\totalnonincl}_{\genemb}  \label{eq:nondominance1}\\
\weaklycontains^{\gennonincl}_{\genemb} \; \not\dom \; \stronglycontains^{\gennonincl}_{\genemb} \label{eq:nondominance2}\\
\weaklycontains^{\partialnonincl}_{\softemb} \; \not\dom \; \weaklycontains^{\partialnonincl}_{\strictemb} \text{ and } \stronglycontains^{\partialnonincl}_{\softemb} \; \not\dom \; \stronglycontains^{\partialnonincl}_{\strictemb} \label{eq:nondominance3}\\
\weaklycontains^{\partialnonincl}_{\strictemb} \; \not\dom \; \stronglycontains^{\partialnonincl}_{\softemb} \label{eq:nondominance4}
\\
\stronglycontains^{\partialnonincl}_{\softemb} \; \not\dom \;\weaklycontains^{\partialnonincl}_{\strictemb} \label{eq:nondominance5}\\
\weaklycontains^{\gennonincl}_{\strictemb} \; \not\dom \;\stronglycontains^{\gennonincl'}_{\strictemb} \label{eq:nondominance6}\\
\weaklycontains^{\totalnonincl}_{\softemb} \; \not\dom \;\stronglycontains^{\partialnonincl}_{\softemb} \label{eq:nondominance7}\\
\stronglycontains^{\partialnonincl}_{\strictemb} \; \not\dom \;\weaklycontains^{\totalnonincl}_{\strictemb} \label{eq:nondominance8}
\end{gather}}
where $\gennoninclrel\in\left\{\totalnoninclrel,\partialnoninclrel\right\}$ and $\genemb\in\left\{\softemb,\strictemb\right\}$.
\end{proposition}

Proposition \ref{prop:dominances} gathers results from Section \ref{sec:negpatterns:semantics}. 
Each line expresses several relationships between pairs of containment relations. Equations \ref{eq:dominance1}-\ref{eq:dominance4} are dominance statements deduced from Lemmas \ref{lemma:bulletimpliescirc}, \ref{prop:sqsubset_eqembeddings}, \ref{lemma:strictocc_implies_softocc} and \ref{lemma:notstrict_implies_notsoft}. Equations \ref{eq:nondominance1}-\ref{eq:nondominance3} state the absence of dominance for which we can exhibit counterexamples. In addition, many other dominance and non dominance relationships can be deduced from Proposition \ref{prop:dominances} using transitivity of dominance (Lemma \ref{lemma:dom_preorder}). 
Table \ref{tab:dominances} summarizes them. 

\begin{table*}
\caption{Dominance. $\dom$ (resp.\ $\nodom$) means that the semantics at the left of the row dominates (resp.\ does not dominate) the semantics at the top of the column.}
\label{tab:dominances}

\centering
\begin{tabular}{c@{\hspace*{2em}}c@{\hspace*{2em}}c@{\hspace*{2em}}c@{\hspace*{2em}}c@{\hspace*{2em}}c@{\hspace*{2em}}c@{\hspace*{2em}}c@{\hspace*{2em}}c}
\toprule
& $\stronglycontains^{\partialnonincl}_{\strictemb}$ &
$\weaklycontains^{\partialnonincl}_{\strictemb}$ &
$\stronglycontains^{\partialnonincl}_{\softemb}$ &
$\weaklycontains^{\partialnonincl}_{\softemb}$ &
$\stronglycontains^{\totalnonincl}_{\strictemb}$ &
$\weaklycontains^{\totalnonincl}_{\strictemb}$ &
$\stronglycontains^{\totalnonincl}_{\softemb}$ &
$\weaklycontains^{\totalnonincl}_{\softemb}$ \\[-0.25ex]
 \midrule
$\stronglycontains^{\partialnonincl}_{\strictemb}$ &
 $\cdot$ & $\dom$ & $\dom$ & $\dom$ & $\nodom$ & $\nodom$ & $\nodom$ & $\nodom$ \\[1ex] 
$\weaklycontains^{\partialnonincl}_{\strictemb}$ &
 $\nodom$ & $\cdot$ & $\nodom$ & $\dom$ & $\nodom$ & $\nodom$ & $\nodom$ & $\nodom$ \\[1ex] 
$\stronglycontains^{\partialnonincl}_{\softemb}$ &
 $\nodom$ & $\nodom$ & $\cdot$ & $\dom$ & $\nodom$ & $\nodom$ & $\nodom$ & $\nodom$\\[1ex] 
$\weaklycontains^{\partialnonincl}_{\softemb}$ &
 $\nodom$ & $\nodom$ & $\nodom$ & $\cdot$ & $\nodom$ &$\nodom$ & $\nodom$ & $\nodom$\\[1ex] 
$\stronglycontains^{\totalnonincl}_{\strictemb}$ &
 $\dom$ & $\dom$ & $\dom$ & $\dom$ & $\cdot$ & $\dom$ & $\dom$ & $\dom$ \\[1ex]
$\weaklycontains^{\totalnonincl}_{\strictemb}$ &
 $\nodom$ & $\dom$ & $\nodom$ & $\dom$ & $\nodom$ & $\cdot$ & $\nodom$ & $\dom$ \\[1ex] 
$\stronglycontains^{\totalnonincl}_{\softemb}$ &
 $\dom$ & $\dom$ & $\dom$ & $\dom$ & $\dom$ & $\dom$ & $\cdot$ & $\dom$ \\[1ex]
$\weaklycontains^{\totalnonincl}_{\softemb}$ &
 $\nodom$ & $\dom$ & $\nodom$  & $\dom$ & $\nodom$ & $\dom$ & $\nodom$ & $\cdot$ \\ 
\bottomrule
\end{tabular}
\end{table*}

An interesting result in Proposition \ref{prop:dominances} is that there are two pairs of containment relations,  $\left(
\stronglycontains^{\totalnonincl}_{\softemb}, \stronglycontains^{\totalnonincl}_{\strictemb}\right)$ and $\left(\weaklycontains^{\totalnonincl}_{\softemb}, \weaklycontains^{\totalnonincl}_{\strictemb}\right)$, whose two members are equivalent. 
In fact, there are six equivalence classes of containment relations: 
$\left\{\stronglycontains^{\partialnonincl}_{\softemb}\right\}$, $\left\{\stronglycontains^{\partialnonincl}_{\strictemb}\right\}$,
$\left\{\weaklycontains^{\partialnonincl}_{\softemb}\right\}$, $\left\{\weaklycontains^{\partialnonincl}_{\strictemb}\right\}$,
$\left\{\stronglycontains^{\totalnonincl}_{\softemb}, \stronglycontains^{\totalnonincl}_{\strictemb}\right\}$ and $\left\{\weaklycontains^{\totalnonincl}_{\softemb}, \weaklycontains^{\totalnonincl}_{\strictemb}\right\}$.
Figure \ref{fig:dominances} illustrates the dominance relation on the quotient set $\Theta/\!\!\sim$.

We can finally point out that Lemma \ref{lemma:soft_implies_strict} adds a dominance relationship for the case that, in negative sequential patterns, negative itemsets are restricted to be singleton sets. In this case, the equivalence classes become: 
$\left\{\weaklycontains^{\totalnonincl}_{\softemb}, \weaklycontains^{\totalnonincl}_{\strictemb}\right\}$,
$\left\{\weaklycontains^{\partialnonincl}_{\softemb},\weaklycontains^{\partialnonincl}_{\strictemb}\right\}$,
$\left\{\stronglycontains^{\totalnonincl}_{\softemb}, \stronglycontains^{\totalnonincl}_{\strictemb}\right\}$ 
and $\left\{\stronglycontains^{\partialnonincl}_{\softemb}, \stronglycontains^{\partialnonincl}_{\strictemb}\right\}$.
Figure \ref{fig:dominances} illustrates the dominance relation on the quotient set \mbox{$\Theta/\!\!\sim$} in this specific case. 

\begin{figure}[tbh]
\caption[Dominance]{Dominance between containment relations. The labels for edges refer to the corresponding equations in Proposition \ref{prop:dominances}.\hspace*{0.5em}Dominance goes from top to bottom i.e.\ $\stronglycontains^{\totalnonincl}_{\softemb}$ as well as $\stronglycontains^{\totalnonincl}_{\strictemb}$ dominate all other containment relations.}
\label{fig:dominances}
\centering
\scalebox{0.9}{
\begin{tikzpicture}  
%noeuds, ordonnes
    \node (n0) at (2,10) {{\Large$\stronglycontains^{\totalnonincl}_{\softemb} \enskip\sim\enskip \stronglycontains^{\totalnonincl}_{\strictemb}$}};
    \node (n1) at (0,8) {{\Large$\weaklycontains^{\totalnonincl}_{\softemb} \enskip\sim\enskip \weaklycontains^{\totalnonincl}_{\strictemb}$}};
    \node (n2) at (2,6) {{\Large$\weaklycontains^{\partialnonincl}_{\strictemb}$}};
    \node (n3) at (2,3.5) {{\Large$\weaklycontains^{\partialnonincl}_{\softemb}$}}; 
    \node (n4) at (4,7.75) {{\Large$\stronglycontains^{\partialnonincl}_{\strictemb}$}}; 
    \node (n5) at (4,5) {{\Large$\stronglycontains^{\partialnonincl}_{\softemb}$}};

%relations
\draw (n0) to node[left] {(\ref{eq:dominance3})$\enskip$} (n1);
\draw (n1) to node[left] {(\ref{eq:dominance4})$\enskip$} (n2);
\draw (n2) to node[left] {(\ref{eq:dominance1})} (n3);
\draw (n0) to node[right] {(\ref{eq:dominance4})} (n4);
\draw (n4) to node[right] {(\ref{eq:dominance1})} (n5);
\draw (n4) to node[right] {$\enskip$(\ref{eq:dominance3})} (n2);
\draw (n5) to node[right] {$\enskip$(\ref{eq:dominance3})} (n3);

\end{tikzpicture}
}
\end{figure}

\begin{figure}[tbh]
\caption[Dominance2]{Dominance between containment relations for the case that, in negative sequential patterns, negative itemsets are restricted to be singleton sets. The labels for edges again refer to the equations given in Proposition \ref{prop:dominances}.}
\label{fig:dominances_withoutitemsets}
\centering
\scalebox{0.9}{
\begin{tikzpicture}  
%noeuds, ordonnes
    \node (n0) at (2,10) {{\Large$\stronglycontains^{\totalnonincl}_{\softemb} \enskip\sim\enskip \stronglycontains^{\totalnonincl}_{\strictemb}$}};
    \node (n1) at (0,8) {{\Large$\weaklycontains^{\totalnonincl}_{\softemb} \enskip\sim\enskip \weaklycontains^{\totalnonincl}_{\strictemb}$}};
    \node (n2) at (2,6) {{\Large$\weaklycontains^{\partialnonincl}_{\strictemb} \enskip\sim\enskip \weaklycontains^{\partialnonincl}_{\softemb}$}};
    \node (n4) at (4,7.75) {{\Large$\stronglycontains^{\partialnonincl}_{\strictemb} \enskip\sim\enskip \stronglycontains^{\partialnonincl}_{\softemb}$}}; 

%relations
\draw (n0) to node[left] {(\ref{eq:dominance3})$\enskip$} (n1);
\draw (n1) to node[left] {(\ref{eq:dominance4})$\enskip$} (n2);
\draw (n0) to node[right] {(\ref{eq:dominance4})} (n4);
\draw (n4) to node[right] {$\enskip$(\ref{eq:dominance3})} (n2);

\end{tikzpicture}
}
\end{figure}
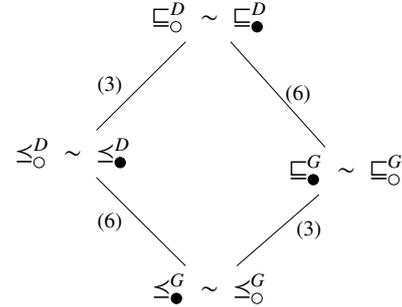

\section{Anti-monotonicity}
It is now time to check whether there are containment relations that enjoy interesting properties. 
In our initial context of mining frequent negative sequential patterns, we investigate anti-monotonicity properties.

According to Wang et \al \cite{Wang2019}, ``the downward property (\ldots) does not hold in negative sequential analysis''. The ``downward property'' denotes the anti-monotonicity property. We will see that this assertion is actually false with some semantics.

Anti-monotonicity makes sense only with a partial order on the set of NSPs. 
We first introduce different possible partial orders and then we introduce anti-monotonicity.

In the remaining of the section, non-inclusion of itemsets is total non-inclusion, $\totalnoninclrel$.
Thus, we can count on the anti-monotonicity of non-inclusion of itemsets: $q\subseteq q' \implies \forall p\in\mathcal{I},\ (q' \totalnoninclrel p \Rightarrow q \totalnoninclrel p)$ for all itemsets $q,q'$.

\subsection{Partial orders}
Definition~\ref{def:nsp_relations} introduces three relations between negative sequential patterns that are partial orders (see Proposition~\ref{prop:partialorders}).

\begin{definition}[NSP relations]\label{def:nsp_relations}
Consider two NSPs $\seq{p} = \langle p_1 $ $\neg q_1\ p_2\ \neg q_2\ \cdots\  p_{k-1}\ \neg q_{k-1}\ p_{k}\rangle$ and $\seq{p}' =\langle p'_1\ \neg q'_1\ p'_2\ \neg q'_2\ \cdots $ $ p'_{k'-1}\ \neg q'_{k'-1}\ p'_{k'}\rangle$. 

By definition, $\seq{p} \nspincl \seq{p}'$ iff $k\leq k'$ and there exists an increasing\cref{fn:increasing} tuple $(u_i)_{i\in[k]}\in[k']^k$ and:
\begin{enumerate}
\item $\forall i \in [k],\; p_i\subseteq p'_{u_i}$
\item $\forall i \in [k-1],\; q_i\subseteq \bigcup_{j\in[u_i,u_{i+1}-1]} q'_{j}$
\item $k= k' \implies \exists j \in[k],\; p_j \neq p'_{j}$ or $\exists j \in[k-1],\; q_j \neq q'_{j}$
\end{enumerate}
by definition, $\seq{p}\lhd\seq{p}'$ iff $k\leq k'$ and:
\begin{enumerate}
\item $\forall i \in [k],\; p_i\subseteq p'_i$
\item $\forall i \in [k-1],\; q_i\subseteq q'_i$
\item $k= k' \implies p_k\neq p'_{k}$ or $\exists j \in [k-1]$ s.t.\ $q_j\neq q'_j$
\end{enumerate}
and, by definition, $\seq{p}\nspinclplus\seq{p}'$ iff $k = k'$ and:
\begin{enumerate}
\item $\forall i \in [k],\; p_i= p'_i$
\item $\forall i \in [k-1],\; q_i\subseteq q'_i$
\item $\exists j \in [k-1]$ s.t.\ $q_j\neq q'_j$
\end{enumerate}
\end{definition}

The $\nspincl$ relation can be seen as the ``classical'' inclusion relation between sequential patterns \cite{Mooney:2013}. 
An NSP $\seq{p}$ is less specific than $\seq{p}'$ iff $\seq{p}^+$ is a subsequence of $\seq{p}'^+$ and negative constraints are satisfied. 
The main difference with $\lhd$ is that $\nspincl$ permits to insert new positive itemsets in the middle of the sequence while $\lhd$ permits only insertion of new positive itemsets at the end.\footnote{In sequential pattern mining, it is called a \textit{backward}-extension of the patterns.}$^,$\footnote{We remind that, by Definition \ref{def:negativepattern}, $p_i\neq\emptyset$ and that we never have two successive negative itemsets in an NSP.} 
Nonetheless, it is still possible to insert items to the positive itemsets. 
The $\nspinclplus$ does not even permit such differences: for two NSPs to be comparable via $\nspinclplus$, they must have the same positive itemsets.

\begin{lemma}\label{lemma:partialorders_relations}
For $\seq{p},\seq{p}'\in\mathcal{N}$, 
\begin{equation}
\seq{p}\nspinclplus\seq{p}' \implies \seq{p}\lhd\seq{p}' \implies \seq{p}\nspincl\seq{p}'
\end{equation}
\end{lemma}

\begin{proposition}[Strict partial orders]\label{prop:partialorders}
 $\nspincl$, $\lhd$ and $\nspinclplus$  are partial orders on $\mathcal{N}$.
\end{proposition}

We can notice that the third conditions in Definition \ref{def:nsp_relations} enforce the relations to be irreflexive. Removing these conditions enables to define non-strict partial orders.

\subsection{Anti-monotonicity}

Let us first define the anti-monotonicity property of a  containment relation $\theta\in\Theta$ considering a strict partial order $\ltimes\in\{\lhd,\nspincl,\nspinclplus\}$. 

\begin{definition}[Anti-monotonicity on $(\mathcal{N},\ltimes)$]\label{def:antimonotonicity}
Let $\theta\in\Theta$ be a containment relation,
$\theta$ is anti-monotonic on $(\mathcal{N},\ltimes)$ iff for all $\seq{p},\seq{p}'\in \mathcal{N}$ and all sequences $\seq{s}$:
$$\seq{p}\ltimes\seq{p}' \implies \left(\seq{p}'\theta \seq{s} \implies \seq{p}\theta \seq{s}\right)$$
\end{definition}

First of all, we provide an example showing that none of the containment relations is 
anti-monotonic on $(\mathcal{N},\nspincl)$.
Let $\seq{p}=\langle b\ \neg c \ a\rangle$, $\seq{p}'=\langle b\ \neg c \ d\ a\rangle$ and $\seq{s}=\langle b\ e\ d\ c\ a\rangle$. 
Then, we have $\seq{p}\nspincl \seq{p}'$.\footnote{In this case, we do not have $\seq{p}\lhd \seq{p}'$ nor $\seq{p}\nspinclplus \seq{p}'$}
Nonetheless, for each $\theta\in\Theta$, $\seq{p}'\theta\seq{s}$ but it is false that $\seq{p}\theta\seq{s}$. 

In fact, the presence of the item $d$ in the sequence changes the scope for checking the absence of $c$. 
This example is similar to the one used by Zheng et \al \cite{zheng:2009:negative} to state that anti-monotonic property does not hold for negative sequential patterns.
Nonetheless, the anti-monotonicity property holds in case the partial order prevents from changing the scope for absent items.

\begin{proposition}\label{prop:antimon_lhd}
$\weaklycontains^{\totalnonincl}_{\softemb}$ and $\weaklycontains^{\totalnonincl}_{\strictemb}$ are anti-monotonic on $(\mathcal{N},\lhd)$.
\end{proposition}

Proposition \ref{prop:antimon_lhd} shows that using the $\lhd$ partial order causes anti-monotonicity to hold for containment with weak-occurrence. It is not the case with strong-occurrence, though. 
Let us give a counterexample illustrating what can happen with strong-occurrence. 
Let $\seq{p}=\langle a\ \neg b \ c\rangle$, $\seq{p}'=\langle a\ \neg b \ c\ d\rangle$ and $\seq{s}=\langle a\ c\ d\ a\ b\ c\rangle$. 
Then, we have $\seq{p}\lhd \seq{p}'$.\footnote{In this case, we also have $\seq{p}\nspincl \seq{p}'$ 
(see Lemma \ref{lemma:partialorders_relations}) but not $\seq{p}\nspinclplus \seq{p}'$}
Nonetheless, $\seq{p}'\stronglycontains^{\totalnonincl}_{\genemb}\seq{s}$ holds but it is false that $\seq{p}\stronglycontains^{\totalnonincl}_{\genemb}\seq{s}$. In fact, without the presence of the item $d$ in the pattern, there are three possible embeddings of $\seq{p}$ in $\seq{s}$. For $\stronglycontains^{\totalnonincl}_{\genemb}$ each embedding must satisfy the negation of $b$, which is not the case, but for $\weaklycontains^{\totalnonincl}_{\genemb}$ it is sufficient to have only one embedding satisfying negations.

The previous example illustrates the problem when extending the pattern with additional itemsets. The same issue is encountered with the following example considering patterns of equal length while one pattern has an extended itemset. Let $\seq{p}=\langle a\ \neg b \ c\rangle$, $\seq{p}'=\langle a\ \neg b \ (cd)\rangle$ and $\seq{s}=\langle a\ (cd)\ a\ b\ c\rangle$. Then, we have $\seq{p}\lhd \seq{p}'$. Nonetheless, $\seq{p}'\stronglycontains^{\totalnonincl}_{\genemb}\seq{s}$ holds but it is false that $\seq{p}\stronglycontains^{\totalnonincl}_{\genemb}\seq{s}$.

\begin{proposition}\label{prop:antimon_nspinclplus}
$\weaklycontains^{\totalnonincl}_{\softemb}$, $\weaklycontains^{\totalnonincl}_{\strictemb}$, $\stronglycontains^{\totalnonincl}_{\softemb}$ and $\stronglycontains^{\totalnonincl}_{\strictemb}$ are anti-monotonic on $(\mathcal{N},\nspinclplus)$.
\end{proposition}

We remind that this section was restricted to the case of total non-inclusion ($\totalnoninclrel$) but the results also hold when $\totalnoninclrel$ is replaced by $\partialnoninclrel$ except that we must reverse the inclusion relations for negatives in the partial orders (that is, $\bigcup_{j\in[u_i,u_{i+1}-1]} q'_{j} \subset q_i$ for $\nspincl$ and $q'_i\subseteq q_i$ for $\lhd$ and $\nspinclplus$).

\section{Application to pattern mining}
The definitions of pattern support, frequent pattern and pattern mining derive naturally from the notion of occurrence of a negative sequential pattern, no matter the choices for embedding (soft or strict), non-inclusion (partial or total) and occurrences (weak or strong).
However, these choices about the semantics of NSPs impact directly the number of frequent patterns (under the same minimal threshold constraint) and also computation time.
The stronger the negative constraints, the fewer the number of sequences containing a pattern, and the lesser the number of frequent patterns.

\begin{definition}[Pattern supports]
Let $\mathcal{D}=\{\seq{s}_i\}_{i\in[n]}$ be a dataset of sequences and $\seq{p}$ be an NSP.
The support of $\seq{p}$ in $\mathcal{D}$, denoted %%%$supp_{\theta}^{\mathcal{D}}(p)$ 
$supp\!\!-\!\!{\theta}^\mathcal{D}(p)$, 
is the number of sequences of $\mathcal{D}$ in which $\seq{p}$ occurs according to the $\theta\in\Theta$ containment relation.
\end{definition}

When there is no ambiguity on the dataset of sequences, $supp\!\!-\!\!{\theta}^\mathcal{D}(p)$ is denoted $supp\!\!-\!\!{\theta}(p)$. 

Clearly, if a containment relation $\theta$ is dominated by another containment relation $\theta'$, then the support of the pattern evaluated with $\theta$ is lower than the support of the pattern evaluated with $\theta'$. The next proposition ensues from Proposition~\ref{prop:dominances}.

\begin{proposition}
\label{prop:support_dominances}
For $\seq{p}\in\mathcal{N}$,
\begin{align}
supp-\!\stronglycontains^{\gennonincl}_{\genemb}(\seq{p}) & \leq supp-\!\weaklycontains^{\gennonincl}_{\genemb}(\seq{p})
\label{eq:support_dominance3}
\end{align}
\vspace*{-4ex}
\begin{align}
\begin{split}
supp-\!\stronglycontains^{\gennonincl}_{\strictemb}(\seq{p}) & \leq supp-\!\stronglycontains^{\gennonincl}_{\softemb}(\seq{p})\\[-0.5ex]
supp-\!\weaklycontains^{\gennonincl}_{\strictemb}(\seq{p}) & \leq supp-\!\weaklycontains^{\gennonincl}_{\softemb}(\seq{p}) 
\end{split}
\end{align}
\vspace*{-2ex}
\begin{align}
\begin{split}
supp-\!\stronglycontains^{\totalnonincl}_{\softemb}(\seq{p}) & \leq supp-\!\stronglycontains^{\totalnonincl}_{\strictemb}(\seq{p})\\[-0.5ex]
supp-\!\weaklycontains^{\totalnonincl}_{\softemb}(\seq{p}) & \leq supp-\!\weaklycontains^{\totalnonincl}_{\strictemb}(\seq{p})
\end{split}
\end{align}
\vspace*{-2ex}
\begin{align}
\begin{split}
supp-\!\stronglycontains^{\totalnonincl}_{\genemb}(\seq{p}) & \leq supp-\!\stronglycontains^{\partialnonincl}_{\genemb}(\seq{p})\\[-0.5ex]
supp-\!\weaklycontains^{\totalnonincl}_{\genemb}(\seq{p}) & \leq supp-\!\weaklycontains^{\partialnonincl}_{\genemb}(\seq{p})
\end{split}
\end{align}
\end{proposition}

In addition, the following anti-monotonicity properties of support measures ensue from Propositions \ref{prop:antimon_lhd} and \ref{prop:antimon_nspinclplus}.
\begin{proposition}
\label{prop:support_antimonotonie}
For $\seq{p}, \seq{p}'\in\mathcal{N}$,
\begin{align}
\seq{p} \lhd \seq{p}' & \implies
supp-\!\weaklycontains^{\totalnonincl}_{\genemb}(\seq{p}') \leq supp-\!\weaklycontains^{\totalnonincl}_{\genemb}(\seq{p})
\label{eq:support_antimonotonie1}
\end{align}
\vspace*{-2ex}
\begin{align}
\seq{p} \nspinclplus \seq{p}' & \implies
\left\{\begin{array}{l}
supp-\!\stronglycontains^{\totalnonincl}_{\genemb}(\seq{p}') \leq supp-\!\stronglycontains^{\totalnonincl}_{\genemb}(\seq{p})
\\[0.5ex]
supp-\!\weaklycontains^{\totalnonincl}_{\genemb}(\seq{p}') \leq supp-\!\weaklycontains^{\totalnonincl}_{\genemb}(\seq{p})
\end{array}\right.
\label{eq:support_antimonotonie2}
\end{align}
\end{proposition}

There are two practical ways to exploit these results to implement efficient frequent NSP mining algorithms. On the one hand, the results from Proposition \ref{prop:support_antimonotonie} can be directly used to implement algorithms with efficient and correct strategies to prune the search space.\footnote{Completeness and non-redundancy of algorithms are out of the scope of this article.} 
For $\weaklycontains^{\totalnonincl}_{\genemb}$ containment relation, Equation \ref{eq:support_antimonotonie1} exploits the $\lhd$ partial order to early prune a priori unfrequent patterns.
For $\stronglycontains^{\totalnonincl}_{\genemb}$ containment relation, the $\nspinclplus$ partial order must be used to ensure the correctness of the algorithm (Equation \ref{eq:support_antimonotonie2}). Unfortunately, $\nspinclplus$ is less interesting than $\lhd$ because there are fewer pairs of comparable patterns.
On the other hand, the support evaluated with $\weaklycontains^{\totalnonincl}_{\genemb}$ is an upper bound for the support of $\stronglycontains^{\totalnonincl}_{\genemb}$ (Equation \ref{eq:support_dominance3}). Thus, it is possible also to prune patterns accessible with the partial order $\lhd$ without losing the correctness of the pruning strategy.

\section{A proposal to disambiguate syntax of negative sequential patterns}
The $\neg$ symbol is overloaded in the literature about negative sequential pattern mining. 
Our intuition was that the different approaches \cite{cao2016nsp,Gong2017eNSPFI,Guyet2020,hsueh:2008:PNSP,zheng:2009:negative} do not extract the same set of patterns because of slightly different definition of negative patterns. 
Our framework deals with the need to define an unambiguous containment relation between a negative sequential pattern $\seq{p}$ and a sequence $\seq{s}$ that informs the user about:
\begin{itemize}
\item how multiple occurrences of the positive part of $\seq{p}$ are handled,
\item how negative itemsets are handled (type of embedding and type of non-inclusion relation).
\end{itemize}
We separate these two dimensions of our definition of a containment relation because the second refers to single itemsets, while the first refers to the whole pattern.

Thanks to our framework, we are able to assign a containment relation to each approach from the literature. The approaches based on eNSP \cite{cao2016nsp,Gong2017eNSPFI} are based on the containment relation $\stronglycontains^{\totalnonincl}_{\strictemb}$, PNSP \cite{hsueh:2008:PNSP} uses the relation $\weaklycontains^{\partialnonincl}_{\strictemb}$, and NegPSpan \cite{Guyet2020} and NegGSP \cite{zheng:2009:negative}
deal with the equivalent relations $\weaklycontains^{\totalnonincl}_{\strictemb}$ and $\weaklycontains^{\totalnonincl}_{\softemb}$.

This confirms our initial intuition: the different approaches do not use the same containment relations and thus they do not aim at extracting the same set of patterns. 
Moreover, it is worth noticing that these four approaches explore a large range of the possible containment relations. 
eNSP exploits the strong notion of occurrence while the other approaches exploit the weak notion. 
All but PNSP approaches are based on total non-inclusion. 
Strict-embedding ($\strictemb$) is generally preferred to soft-embedding ($\softemb$).  
eNSP ($\stronglycontains^{\totalnonincl}_{\strictemb}$) made the most restrictive choice by using the containment relation that dominates all the others. 
On the opposite, the two least restrictive choices ($\weaklycontains^{\partialnonincl}_{\softemb}$ and $\stronglycontains^{\partialnonincl}_{\softemb}$) have not been explored, presumably due to their obvious lack of suitable properties for pattern mining. 

Finally, 
it is worth comparing negative sequential patterns with some formulas in Linear Temporal Logic on finite traces (LTLf) \cite{deGiacomo2013}.
The question is to find specific LTLf formulas capturing our containment relations between any two patterns $\seq{p}$ and $\seq{s}$.
Then, it is interesting to notice that containment relations based on soft-embedding have simple counterparts in the language of LTLf.
Indeed, the soft-embedding constraint imposes each successive itemset of a sequence to not contain some negated items. 
The strict-embedding constraint, which requires to evaluate a union of items, does not fit well to the linearity of LTLf formulas.

\espace

\section{Conclusions and perspectives}
In this article, we investigated formal properties of the semantics of negation in sequential patterns to answer our two main questions. 
\begin{enumerate}
\item %
\emph{What is a proper support measure for negative sequential patterns?} 
We gave eight possible semantics and as many support measures. We can conclude that there is not a single way to evaluate the support of an NSP.
\item 
\emph{Is there a support measure enjoying anti-monotonicity?} 
We run counter the state of the art by proposing three partial orders for which anti-monotonicity holds although only for some semantics of negative sequential pattern mining.
\end{enumerate}

The combination of partial order $\lhd$ and containment relation $\weaklycontains^{\totalnonincl}_{\genemb}$ appears to be a good candidate for developping a complete, correct and non-redundant negative sequential pattern mining algorithm \cite{Guyet2020}. 
One advantage of an approach based on an anti-monotonic support measure is the benefit from decades of research in pattern mining so as to extend the mining of NSP to the mining of closed NSP or the mining of NSP with \textit{maxgap} or \textit{maxspan} constraints. 

Nonetheless, no semantics is ``more correct'' or relevant than another one. It depends on the notion to be captured. Our objective is to give the opportunity to make an educated choice. 
It is especially important with NSP as the choice of a mining algorithm is not only a matter of computational efficiency, but also a matter of semantics.

In view of Definition \ref{def:NSP_embedding} and Lemma \ref{prop:sqsubset_eqembeddings}, three possibilities arise for evaluating a negative itemset (syntactically distinguished below by writing a negative itemset $\neg (a_1, \dots, a_{l_i})$ or $\neg \{a_1, \dots, a_{l_i}\}$ or $\neg |a_1, \dots, a_{l_i}|$) as follows:
\begin{description}
\item[$\neg (a_1, \dots, a_{l_i})$] is evaluated as
\[\mbox{$\{a_1, \dots, a_{l_i}\} \partialnoninclrel s_j,\;\forall j\in [e_{i}+1,e_{i+1}-1]$ for all $i\in[m-1]$}\] Intuitively, you check that, in between $s_{e_i}$ (i.e., a match for $p_i$) and $s_{e_{i+1}}$ (i.e., a match for $p_{i+1}$), none of these $s_{j}$ include all of $a_1, \dots, a_{l_i}$.\vspace*{2.5ex}
\item[$\neg \{a_1, \dots, a_{l_i}\}$] is evaluated as
\[\mbox{$\{a_1, \dots, a_{l_i}\} \partialnoninclrel \bigcup_{j\in [e_{i}+1,e_{i+1}-1]} s_j$ for all $i\in[m-1]$}\]
Intuitively, you check that there exists some item in $a_1, \dots, a_{l_i}$ that does not occur at all in between $s_{e_i}$ (i.e., a match for $p_i$) and $s_{e_{i+1}}$ (i.e., a match for $p_{i+1}$).\vspace*{2.5ex}
\item[$\neg |a_1, \dots, a_{l_i}|$] is evaluated as
\[\mbox{$\{a_1, \dots, a_{l_i}\} \totalnoninclrel \bigcup_{j\in [e_{i}+1,e_{i+1}-1]} s_j$ for all $i\in[m-1]$}\]
Intuitively, you check that every item in $a_1, \dots, a_{l_i}$ fails to occur in between $s_{e_i}$ (\ie, a match for $p_i$) and $s_{e_{i+1}}$ (i.e., a match for $p_{i+1}$).
\end{description}

This opens the way for a syntax of negative sequential patterns that is even more expressive. 
Indeed, it enables to mix different types of negation within a pattern. For instance, we can specify patterns such as $\langle a\ \neg|bc|\ f\ \neg\{ac\}\ b\rangle$ 
(intuitively: none of $b$ and $c$ occur between $a$ and $f$; also, either $a$ or $c$ (or both) does not occur at all between $f$ and $b$).

The first perspective of this work is to evaluate the proposed notations on a panel of real users. A preliminary survey concluded on a lack of any dominant interpretation of the $\neg$ symbol. We would like to confirm this preliminary result on a larger panel and to evaluate the benefit of having a dedicated syntax for each containment relation.

Our second perspective is to extend our theoretical results from the pattern recognition perspective. 
Matching sequential patterns in sequences is a fundamental issue in monitoring of discrete event systems, in genetic data analysis, in text analysis, etc. Adding negations to sequential patterns increases the expressivity of the pattern language. It raises questions about space and time complexity of the recognition and/or enumeration of negative sequential patterns: are the different containment relations equally hard to evaluate in sequences?

\bibliography{biblio}

\newpage
\section*{Proofs}

\begin{proof}[Proof of Lemma \ref{lemma:non_incl_relation}]
Let $P,I\subseteq\mathcal{I}$ s.t. $P\totalnoninclrel I$. If $P=\emptyset$, by definition, $P\partialnoninclrel I$. Otherwise, because $P$ is not empty, then there exists $e\in P$ s.t. $e\not\in I$, \ie $P\partialnoninclrel I$.
\end{proof}

\begin{proof}[Proof of Lemma \ref{lemma:bulletimpliescirc}]
Let $\seq{e}=(e_i)_{i\in[m]}\in [n]^m$ be a $\strictemb$-embedding of a NSP $\seq{p}=\langle p_1\ \neg q_1\ \dots\ \neg q_{m-1}\ p_m\rangle$ in a sequence $\seq{s}=\langle s_1\ \dots\ s_n\rangle$. 
For all positive itemsets $p_i$, the definition of $\strictemb$-embedding matches the one for $\softemb$-embedding. 
For a negative itemset $q_i$, let us start with $\partialnoninclrel\myeq\gennoninclrel$. By definition \ref{def:NSP_embedding}, 
$q_i \partialnoninclrel \bigcup_{j\in [e_{i}+1,e_{i+1}-1]}s_j$, and by Definition \ref{def:IS_notincluded}, $\exists \alpha\in q_i,\; \alpha\notin \bigcup_{j\in [e_{i}+1,e_{i+1}-1]}s_j$. And then, $\exists \alpha\in q_i,\;\forall j\in [e_{i}+1,e_{i+1}-1],\; \alpha\notin s_j$. That is $\forall j\in [e_{i}+1,e_{i+1}-1],\; \exists \alpha\in q_i,\;\alpha \notin s_j$. This shows $\forall j\in [e_{i}+1,e_{i+1}-1],\; q_i\partialnoninclrel s_j$ ($\softemb$-embedding definition).
It remains $\totalnoninclrel\myeq\gennoninclrel$. By definition \ref{def:NSP_embedding}, 
$q_i \totalnoninclrel \bigcup_{j\in [e_{i}+1,e_{i+1}-1]}s_j$, and by Definition \ref{def:IS_notincluded}, $\forall \alpha\in p_i,\; \alpha\notin \bigcup_{j\in [e_{i}+1,e_{i+1}-1]}s_j$. And then, $\forall \alpha\in q_i,\;\forall j\in [e_{i}+1,e_{i+1}-1],\; \alpha\notin s_j$. That is $\forall j\in [e_{i}+1,e_{i+1}-1],\; \forall \alpha\in q_i,\;\alpha\notin s_j$. This shows $\forall j\in [e_{i}+1,e_{i+1}-1],\; q_i\gennoninclrel s_j$.
\end{proof}

\begin{proof}[Proof of Lemma \ref{prop:sqsubset_eqembeddings}]
Let $\seq{s}=\langle s_1\ \dots\ s_n\rangle$ be a sequence and $\seq{p}=\langle p_1\ \neg q_1\ \dots\ \neg q_{m-1}\ p_m\rangle$ be a negative sequential pattern.
Lemma \ref{lemma:bulletimpliescirc} shows that $\strictemb$-embedding implies $\softemb$-embedding.
It remains the implication to the left.
Let $\seq{e}=(e_i)_{i\in[m]}\in [n]^m$ be a $\softemb$-embedding of pattern $\seq{p}$ in sequence $\seq{s}$. Then, the definition matches the one for $\strictemb$-embedding for positives, $p_i$. For negatives, $q_i$, then $\forall j\in [e_{i}+1,e_{i+1}-1],\; q_i \totalnoninclrel s_j$, \ie $\forall j\in [e_{i}+1,e_{i+1}-1],\; \forall \alpha\in q_i,\; \alpha \notin s_j$ and then $\forall \alpha\in q_i,\; \forall j\in [e_{i}+1,e_{i+1}-1],\; \alpha \notin s_j$.
It thus implies that $\forall \alpha\in q_i,\; \alpha \notin \bigcup_{j\in [e_{i}+1,e_{i+1}-1]} s_j$, \ie by definition, $q_i \totalnoninclrel \bigcup_{j\in [e_{i}+1,e_{i+1}-1]} s_j$.
\end{proof}

\begin{proof}[Proof of Lemma \ref{lemma:soft_implies_strict}]
Let $\seq{s}=\langle s_1\ \dots\ s_n\rangle$ be a sequence and $\seq{p}=\langle p_1\ \neg q_1\ \dots\ \neg q_{m-1}\ p_m\rangle$ be a NSP s.t. $\forall i,\; |q_i|\leq 1$.
Due to Lemma \ref{prop:sqsubset_eqembeddings}, we only need to deal with the case that $\gennoninclrel$ is $\partialnoninclrel$. 
Let $\seq{e}=(e_i)_{i\in[m]}\in[n]^m$ be a $\softemb$-embedding of $\seq{p}$ in $\seq{s}$ then, by definition, 1) $p_i\subseteq s_{e_i}$ for all $i\in[m]$ and 2) $q_i \partialnoninclrel s_{j}$ for all $j\in [e_{i}+1,e_{i+1}-1]$. 
In case $|q_i|=0$, there is no constraint. In case $|q_i|=1$, 
then 2) becomes $q_i\notin s_{j}$ for all $j\in [e_{i}+1,e_{i+1}-1]$.
Hence, $q_i\notin \bigcup_{j\in [e_{i}+1,e_{i+1}-1]}{s_j}$ \ie $q_i \partialnoninclrel \bigcup_{j\in [e_{i}+1,e_{i+1}-1]}{s_j}$.
As a consequence $\seq{e}$ is a $\strictemb$-embedding of $\seq{p}$.
\end{proof}

%%%%%%%%%%%%%%%%%%%%%%%%%%%%%%%%%%%%%%%%%%%%%%%%%%%%
\begin{proof}[Proof of Lemma \ref{lemma:pos_embedding}]
Let $\seq{s}=\langle s_1\ \dots\  s_n\rangle$ be a sequence and $\seq{p}=\langle p_1\ \neg q_1\ \dots\ \neg q_{m-1}\ p_m\rangle \in \mathcal{N}$ be a pattern. 
By definition \ref{def:NSP_embedding}, if $\seq{e}=(e_i)_{i\in[m]}\in [n]^m$ is an embedding of pattern $\seq{p}$ in sequence $\seq{s}$ then $\forall i\in[m]$: $p_{i} \subseteq s_{e_{i}}$ because $p_{i}$ is positive.
The condition $e_i < e_{i+1}$ in Definition \ref{def:positivepattern_embedding} immediately follows from the requirement in Definition \ref{def:NSP_embedding} that $\seq{e}$ be increasing with no repetition.
It ensues that $\seq{e}$ is an embedding of the positive pattern $\seq{p}^+$.
\end{proof}

%%%%%%%%%%%%%%%%%%%%%%%%%%%%%%%%%%%%%%%%%%%%%%%%%%%%
\begin{proof}[Proof of Lemma \ref{lemma:strictocc_implies_softocc}]
Let $\seq{s}=\langle s_1\ \dots\ s_n\rangle$ be a sequence and $\seq{p}=\langle p_1\ \neg q_1\ \dots\ \neg q_{m-1}\ p_m\rangle \in \mathcal{N}$ be a pattern s.t. $p\stronglycontains^{\gennonincl}_{\genemb}s$.
Then, there exists $\seq{e}$ an embedding of $\seq{p}^+$ in $\seq{s}$ and, by definition, it is also an embedding of $\seq{p}$ in $\seq{s}$. 
This means that $p\weaklycontains^{\gennonincl}_{\genemb}s$.
\end{proof}

%%%%%%%%%%%%%%%%%%%%%%%%%%%%%%%%%%%%%%%%%%%%%%%%%%%%
\begin{proof}[Proof of Lemma \ref{lemma:notstrict_implies_notsoft}]
Let $\seq{s}=\langle s_1\ \dots\ s_n\rangle$ be a sequence and $\seq{p}=\langle p_1\ \neg q_1\ \dots\ \neg q_{m-1}\ p_m\rangle \in \mathcal{N}$ be a pattern.

We start by considering relations between semantics at the embedding level, and then we will conclude at the pattern level.

Let's first assume that $\genemb$ is $\softemb$.
Consider a $(\softemb,\totalnonincl)$-embedding 
$e=(e_i)_{i\in[m]}$ of pattern $\seq{p}$ in sequence $\seq{s}$. Hence, $\forall i\in[m]$, $p_i \subseteq s_{e_i}$ and $\forall i\in[m-1]$, $\forall j\in [e_{i}+1,e_{i+1}-1],\, q_i \totalnoninclrel s_j$. According to eq. \ref{eq:non_incl_relation}, we have $q_i \partialnoninclrel s_j$. 
It ensues that $e$ is a $(\softemb,\weaklycontains)$-embedding.

Let's now assume that $\genemb$ is $\strictemb$. 
Let $e=(e_i)_{i\in[m]}$ be a $(\strictemb,\totalnonincl)$-embedding 
of pattern $\seq{p}$ in sequence $\seq{s}$. Hence, $\forall i\in[m]$, $p_i \subseteq s_{e_i}$ and $\forall i\in[m-1]$, $q_i \totalnoninclrel \bigcup_{j\in [e_{i}+1,e_{i+1}-1]} s_j$. In addition $\bigcup_{j\in [e_{i}+1,e_{i+1}-1]} s_j \subseteq \mathcal{I}$, hence $\forall i \in [m-1],\, q_i \partialnoninclrel \bigcup_{j\in [e_{i}+1,e_{i+1}-1]} s_j$ according to eq. \ref{eq:non_incl_relation}. 
It ensues that $e$ is an $(\strictemb,\weaklycontains)$-embedding.

Let's come back to the pattern level. 
Consider $p\gencontains^{\totalnonincl}_{\genemb}s$, in the two cases ($\gencontains\in\{\weaklycontains,\stronglycontains\}$). In the first case the existing $(\genemb,\totalnonincl)$-embedding is a $(\genemb,\partialnonincl)$-embedding, and in the second case, all $(\genemb,\totalnonincl)$-embeddings are $(\genemb,\partialnonincl)$-embeddings. Therefore, we have that $p\gencontains^{\partialnonincl}_{\genemb}s$.
\end{proof}

\begin{proof}[Proof of Lemma \ref{lemma:dom_preorder}]
A pre-order is a reflexive, transitive binary relation. 
The reflexivity of the relation comes with Definition \ref{def:dominance}. 
Let $\theta,\theta', \theta''\in\Theta$ be three dominance relations s.t. $\theta \dom \theta'$ and $\theta' \dom \theta''$. Then, for all $\seq{p}\in\mathcal{N}$ and sequence $\seq{s}$: $\seq{p}\theta \seq{s}\implies \seq{p}\theta' \seq{s}$ and $\seq{p}\theta'\seq{s}\implies \seq{p}\theta''\seq{s}$.
Hence, we have, $\seq{p}\theta\seq{s}\implies \seq{p}\theta''\seq{s}$, \ie $\theta \dom \theta''$.
\end{proof}

\begin{proof}[Proof of Lemma \ref{lemma:equivalent_relation}]
Let $\theta,\theta'\in\Theta$, by reflexivity of $\dom$ we have that $\sim$ is reflexive. By definition ($\theta\dom\theta' \wedge \theta'\dom\theta$), $\sim$ is symmetric. And , $\sim$ is also transitive. Let $\theta,\theta', \theta''\in\Theta$ be three dominance relations s.t.\ $\theta \sim \theta'$ and $\theta' \sim \theta''$ then, $\theta \dom \theta'$, $\theta' \dom \theta''$, $\theta' \dom \theta$ and $\theta'' \dom \theta'$. Hence, by transitivity of $\dom$, $\theta \dom \theta''$ and $\theta'' \dom \theta$, $\theta \sim \theta''$.
\end{proof}

\begin{proof}[Proof of Proposition \ref{prop:dominances}]
Let $\seq{p}\in\mathcal{N}$ and $\seq{s}$ a sequence.

According to Lemma \ref{lemma:strictocc_implies_softocc}, $\seq{p}\stronglycontains^{\gennonincl}_{\genemb}\seq{s} \implies \seq{p}\weaklycontains^{\gennonincl}_{\genemb}\seq{s}$. Thus we obtain Equality \ref{eq:dominance3} by Definition \ref{def:dominance}.

According to Lemma \ref{lemma:notstrict_implies_notsoft}, $\seq{p}\gencontains^{\totalnonincl}_{\genemb}\seq{s} \implies \seq{p}\gencontains^{\partialnonincl}_{\genemb}\seq{s}$. Thus we obtain Equality \ref{eq:dominance4} by Definition \ref{def:dominance}.

According to Lemma \ref{lemma:bulletimpliescirc}, a $\strictemb$-embedding is a $\softemb$-embedding whatever the itemset non-inclusion operator. 
Then, we get $\seq{p}\weaklycontains_{\strictemb}^{\gennonincl}\seq{s} \implies \seq{p}\weaklycontains_{\softemb}^{\gennonincl}\seq{s}$. 
Lemma \ref{lemma:bulletimpliescirc} holds for any $\strictemb$-embedding, hence $\seq{p}\stronglycontains_{\strictemb}^{\gennonincl}\seq{s} \implies \seq{p}\stronglycontains_{\softemb}^{\gennonincl}\seq{s}$.
All this gives 
$\seq{p}\gencontains_{\strictemb}^{\gennonincl}\seq{s} \implies \seq{p}\gencontains_{\softemb}^{\gennonincl}\seq{s}$ (Equality \ref{eq:dominance1}).

In addition, Lemma \ref{prop:sqsubset_eqembeddings} shows that a $\softemb$-embedding is a $\strictemb$-embedding 
(and vice-versa)
in case of total itemset non-inclusion. Then, we can conclude that $\seq{p}\gencontains_{\softemb}^{\totalnonincl}\seq{s} \implies \seq{p}\gencontains_{\strictemb}^{\totalnonincl}\seq{s}$ (Equality \ref{eq:dominance2}).

\espace

Let now gives some counterexamples for known non-dominance relationships. 
For each non-dominance relation, $\theta\not\dom\theta'$, we provide counterexamples for some $\theta,\theta'\in\Theta$, \ie a pattern $\seq{p}$ and a sequence $\seq{s}$ such that $\seq{p}\theta\seq{s}$ but not  $\seq{p}\theta'\seq{s}$.

\begin{itemize}
\item Equation (\ref{eq:nondominance1}): $\theta = (\partialnonincl,\genemb,\gencontains)$, $\theta'=(\totalnonincl,\genemb,\gencontains)$:
Let $\seq{p}=\langle a\ \neg (bc)\ d\rangle$, $\seq{s}=\langle a\ b\ d\rangle$, then $\seq{p}\theta\seq{s}$ but not $\seq{p}\theta'\seq{s}$.
\item Equation (\ref{eq:nondominance2}): $\theta = (\gennonincl,\genemb,\weaklycontains)$, $\theta'=(\gennonincl,\genemb,\stronglycontains)$:
Let $\seq{p}=\langle a\ \neg b\ c\rangle$, $\seq{s}=\langle a\ c\ b\ c\rangle$, then $\seq{p}\theta\seq{s}$ but not $\seq{p}\theta'\seq{s}$.
\item Equation (\ref{eq:nondominance3}): $\theta = (\partialnonincl,\softemb,\gencontains)$, $\theta'=(\partialnonincl,\strictemb,\gencontains)$:
Let $\seq{p}=\langle a\ \neg (bc)\ d\rangle$, $\seq{s}=\langle a\ b\ c\ d\rangle$, then $\seq{p}\theta\seq{s}$ but not $\seq{p}\theta'\seq{s}$.
\item Equation (\ref{eq:nondominance4}): $\theta = (\partialnonincl,\strictemb,\weaklycontains)$, $\theta'=(\partialnonincl,\softemb,\stronglycontains)$:
Let $\seq{p}=\langle a\ \neg (bc)\ d\rangle$, $\seq{s}=\langle a\ b\ d\ c\ d\rangle$, then $\seq{p}\theta\seq{s}$ but not $\seq{p}\theta'\seq{s}$. The strict embedding works for one embedding of the positive partner, but there is a positive partner embedding for which even the soft-embedding.
\item Equation (\ref{eq:nondominance5}): $\theta = (\partialnonincl,\softemb,\stronglycontains)$, $\theta'=(\partialnonincl,\strictemb,\weaklycontains)$:
Let $\seq{p}=\langle a\ \neg (bc)\ d\rangle$, $\seq{s}=\langle a\ b\ c\ d\rangle$, then $\seq{p}\theta\seq{s}$ but not $\seq{p}\theta'\seq{s}$.
\item Equation (\ref{eq:nondominance6}): $\theta = (\gennonincl,\strictemb,\weaklycontains)$, $\theta'=(\gennonincl',\strictemb,\stronglycontains)$:
Let $\seq{p}=\langle a\ \neg b\ c\rangle$, $\seq{s}=\langle a\ c\ b\ c\rangle$, then $\seq{p}\theta\seq{s}$ but not $\seq{p}\theta'\seq{s}$. Redundant with (\ref{eq:nondominance2}) when $\gennonincl'=\gennonincl$.
\item Equation (\ref{eq:nondominance7}): $\theta = (\totalnonincl,\softemb,\weaklycontains)$, $\theta'=(\partialnonincl,\softemb,\stronglycontains)$:
Let $\seq{p}=\langle a\ \neg b\ c\rangle$, $\seq{s}=\langle a\ c\ b\ c\rangle$, then $\seq{p}\theta\seq{s}$ but not $\seq{p}\theta'\seq{s}$.
\item Equation (\ref{eq:nondominance8}): $\theta = (\partialnonincl,\strictemb,\stronglycontains)$, $\theta'=(\totalnonincl,\strictemb,\weaklycontains)$:
Let $\seq{p}=\langle a\ \neg (bc)\ d\rangle$, $\seq{s}=\langle a\ b\ d\rangle$, then $\seq{p}\theta\seq{s}$ but not $\seq{p}\theta'\seq{s}$.
\end{itemize}
\end{proof}

\begin{proof}[Proof of Lemma \ref{lemma:partialorders_relations}]
We start with the implication $\seq{p}\nspinclplus\seq{p}' \implies \seq{p}\lhd\seq{p}'$. 
Let $p,p'\in\mathcal{N}$ s.t. $\seq{p}\nspinclplus\seq{p}'$. By definition, $k=k'$ and 1.\ $\forall i \in [k],\; p_i= p'_i$, 2.\ $\forall i \in [k-1],\; q_i\subseteq q'_i$ and 3.\ $\exists j \in [k-1]$ s.t.\ $q_i\neq q'_i$.
A particular case of 1.\ is that $\forall i \in [k],\; p_i\subseteq p'_i$. In addition, the third condition of $\lhd$ is obtained easily from 3.\ by adding a disjunctive condition. Hence, $\seq{p}\lhd\seq{p}'$.

We now prove the second implication: $\seq{p}\lhd\seq{p}' \implies \seq{p}\nspincl\seq{p}'$. 
Let $p,p'\in\mathcal{N}$ s.t. $\seq{p}\lhd\seq{p}'$. Let's now define the sequence $u_i$ such that $u_i=i$ for all $i\in[k]$. By construction, we have that $u_i < u_{i+1}$, for all $i\in[k-1]$ (see the increasingness requirement in the definition of $\nspincl$). In addition, by definition of $\lhd$, we have that $\forall i \in [k],\; p_i\subseteq p'_{i} =p'_{u_i}$, and $\forall i \in [k-1], q_i\subseteq q'_{i} = \bigcup_{j\in[i,(i+1)-1]}q'_{j} =\bigcup_{j\in[u_i,u_{i+1}-1]}q'_{j}$.
Assuming $k=k'$, then $p_k\neq p'_{k}$ or $\exists j \in [k-1]$ s.t.\ $q_j\neq q'_j$. If $p_k\neq p'_{k}$ the third condition of $\nspincl$ is satisfied (with $j=k$). Otherwise, it is also satisfied with the $j$ of the definition of $\lhd$.
\end{proof}

Note that the proof of the results of Table \ref{tab:dominances} in the article is given at the end of this supplementary material.

\begin{proof}[Proof of Proposition \ref{prop:partialorders}]

We first remind that $\nspincl$ is a strict partial order iff the three following conditions hold: 
\begin{enumerate}
\item $\forall p \in\mathcal{N},$ not $p\nspincl p$ (irreflexive),
\item $\forall p,p',p'' \in\mathcal{N}$, $p \nspincl p'$ and $p' \nspincl p'' \implies p \nspincl p''$ (transitivity),
\item $\forall p,p' \in\mathcal{N}$, $p \nspincl p'\implies $ not $p' \nspincl p$ (antisymmetry)
\end{enumerate}

By Lemma \ref{lemma:partialorders_relations}, if $\nspincl$ is irreflexive and antisymmetric then so are $\lhd$ and $\nspinclplus$.
Hence, we only show that $\nspincl$ is a strict partial order and then show transitivity for $\lhd$ and $\nspinclplus$.

We now prove that $\nspincl$ is a strict partial order.\\
Irreflexivity. Let's assume that $\exists p \in\mathcal{N}$ s.t.\ $p \nspincl p$. Then, identity is the only possibility for $u$, \ie $u_i=i$, for all $i\in[k]$. Then, the third condition implies that $\exists j \in [k-1]$ s.t. $q_j\neq q_j$ or $p_j\neq p_j$, which is absurd. Then $\nspincl$ is irreflexive.\\
Transitivity. Let $p,p',p'' \in\mathcal{N}$ s.t. $p \nspincl p'$ and $p' \nspincl p''$. We denote by $(u_i)_i\in[k']^k$ and $(v_i)_i\in[k'']^{k'}$ the respective mapping, and we define $(w_i)_i\in[k'']^[k]$ such that $w_i = v_{u_i}$ for all $i\in[k]$.
Then, for all $i\in[k]$, $p_i\subseteq p'_{u_i}\subseteq p''_{v_{u_i}}=p''_{w_i}$; $q_i\subseteq \bigcup_{j\in[u_i,u_{i+1}-1]} q'_{j} = \bigcup_{j\in[u_i,u_{i+1}-1]}\bigcup_{l\in[v_{j},v_{j+1}-1]} q''_{l}$. The union of the $q''_{l}$ in the intervals $[v_{j},v_{j+1}-1]$ for $j\in[u_i,u_{i+1}-1]$ can be sum up as an union on the interval $[v_{u_i},v_{(u_{i+1}-1)+1}-1]=[v_{u_i},v_{u_{i+1}}-1]=[w_i,w_{i+1}-1]$ because intervals are contiguous. Then, $q_i\subseteq \bigcup_{j\in[w_i,w_{i+1}-1]} q''_{j}$.
Finally, if $k=k''$, then $k=k''=k'$ and then it exists $j\in[k]$, s.t. $p_j \neq p'_j \subseteq p''_j$ or $q_j \neq q'_j \subseteq q''_j$. 
Thus, $p_j \neq p''_j$ or $q_j \neq q''_j$. As a consequence, we have $p \nspincl p''$.\\
Antisymmetry. Let  $p,p' \in\mathcal{N}$ s.t. $p \nspincl p'$. Then, if $k<k'$ we can not have $p' \nspincl p$. Assuming that $k=k'$ (and thus $u_i=i$ for all $i\in[k]$), we have that there exists $j\in[k]$ s.t. $p_j \neq p'_{j}$ or $q_j \neq q'_{j}$. If 
$p_j \neq p'_{j}$ then, according to 1. $p_j \varsubsetneq p'_{j}$, hence $p'_j \subseteq p_{j}$ fails. If $q_j \neq q'_{j}$, then, according to 2. $q_j \varsubsetneq \bigcup_{j\in[u_i,u_{i+1}-1]} q'_{j} = q'_j$. Thus, it is not possible to have $q'_j \subseteq q_{j} = \bigcup_{j\in[u_i,u_{i+1}-1]} q_j$. As a consequence, we can not have $p' \nspincl p$.

We now turn to $\lhd$.

Transitivity. Let $p,p',p'' \in\mathcal{N}$ s.t. $p \lhd p'$ and $p' \lhd p''$. Then, for all $i\in[k]$, $p_i\subseteq p'_i\subseteq p''_i$ and for all $i\in[k-1]$, $q_i\subseteq q'_i \subseteq q''_i$ ($k\leq k'\leq k''$). 
Finally, if $k=k''$, then $k=k''=k'$. Assuming that $p_k=p'_k$ and $p'_k=p''_k$ then $p_k=p''_{k''}$.
Assuming that $p_k\neq p'_k$ or $p'_k\neq p''_k$, then $\exists j\in[k-1]$ s.t. $q_j \neq q'_j$ or $q'_j \neq q''_j$, and hence $q_j \neq q''_j$. Then, we have that $p\lhd p''$. 
We finish with $\nspinclplus$.

Transitivity. Let $p,p',p'' \in\mathcal{N}$ s.t. $p \nspinclplus p'$ and $p' \nspinclplus p''$. Then, for all $i\in[k]$, $p_i=p'_i=p''_i$ and for all $i\in[k-1]$, $q_i\subseteq q'_i \subseteq q''_i$ ($k= k'= k''$). 
Finally, it is not possible to have $q_i= q''_i$ for all $i\in[k-1]$. In fact, these equalities would entail $q_i = q'_i$ and $q'_i =q''_i$ for all $i\in[k-1]$ because $q_i\subseteq q'_i \subseteq q''_i$. But having all these further equalities is not possible according to 3. Therefore, we have that $p\nspinclplus p''$.
\end{proof}

\begin{proof}[Proof of Proposition \ref{prop:antimon_lhd}]
We start this proof by a small result about the anti-monotonicity of $\totalnoninclrel$. Let $P,Q\in\mathcal{I}$ be two itemsets s.t. $P\subseteq Q$, and $I\in\mathcal{I}$ another itemset. Then, $Q\totalnoninclrel I \implies P\totalnoninclrel I$. In fact, $Q\totalnoninclrel I$ implies that for all $e\in Q,\; e\not\in I$, and because $P\subseteq Q$, we also have that $e\in P,\; e\not\in I$.

Let $\seq{p}=\langle p_1\ \neg q_1\ \dots\ \neg q_{m-1}\ p_m\rangle \in \mathcal{N}$ and $\seq{p}'=\langle p'_1\ \neg q'_1\ \dots\ \neg q'_{m'-1}\ p'_{m'}\rangle \in \mathcal{N}$ be two NSPs s.t.\ $\seq{p}\lhd\seq{p}'$.
 
We first show that an $(\softemb,\totalnonincl)$-embedding of $\seq{p}'$ in a sequence $\seq{s}$, denoted $\seq{e}=(e_i)_{i\in[m']}$, induces an $(\softemb,\totalnonincl)$-embedding of $\seq{p}$.
By Definition \ref{def:NSP_embedding}, we have $p'_i \subseteq s_{e_i},\, \forall i\in[m']$
and $q'_i \totalnoninclrel s_j$, for all $j\in [e_{i}+1,e_{i+1}-1]$ and for all $i\in[m'-1]$.\\
On the other hand, $\seq{p}\lhd\seq{p}'$ implies that $p_i\subseteq p'_i$ for all $i\in[m]$. Then, because $m\leq m'$ ($\seq{p}\lhd\seq{p}'$), we have that $p_i\subseteq s_{e_i}$ for all $i\in[m]$.
In addition, $\seq{p}\lhd\seq{p}'$  also implies that $q_i\subseteq q'_i$ for all $i\in[m-1]$ and thus, by anti-monotonicity of $\totalnoninclrel$ (and $q'_i \totalnoninclrel s_j$), we have $q_i\totalnoninclrel s_j$ for all $j\in [e_{i}+1,e_{i+1}-1]$ and for all $i\in[m-1]$. In conclusion, we have that $\seq{e}=(e_i)_{i\in[m]}$ is an $(\softemb,\totalnonincl)$-embedding of $\seq{p}$.

We now show that an $(\strictemb,\totalnonincl)$-embedding of $\seq{p}'$ in a sequence $\seq{s}$, denoted $\seq{e}=(e_i)_{i\in[m']}$, induces an $(\strictemb,\totalnonincl)$-embedding of $\seq{p}$.
By Definition \ref{def:NSP_embedding}, we have 
$p'_i \subseteq s_{e_i},\, \forall i\in[m']$
and $q'_i \totalnoninclrel \bigcup_{j\in [e_{i}+1,e_{i+1}-1]} s_j$, for all $i\in[m'-1]$.\\
However, $\seq{p}\lhd\seq{p}'$ implies that $p_i\subseteq p'_i$ for all $i\in[m]$. Then, because $m\leq m'$ ($\seq{p}\lhd\seq{p}'$), we have that $p_i\subseteq s_{e_i}$ for all $i\in[m]$.
In addition, $\seq{p}\lhd\seq{p}'$  also implies that $q_i\subseteq q'_i$ for all $i\in[m-1]$, by anti-monotonicity of $\totalnoninclrel$, we have $q'_i\totalnoninclrel \bigcup_{j\in [e_{i}+1,e_{i+1}-1]} s_j$, for all $i\in[m'-1]$. In conclusion, we have that $\seq{e}=(e_i)_{i\in[m]}$ is an $(\strictemb,\totalnonincl)$-embedding of $\seq{p}$.
\end{proof}

\begin{proof}[Proof of Proposition \ref{prop:antimon_nspinclplus}]
Let $\seq{p}=\langle p_1\ \neg q_1\ \dots\ \neg q_{k-1}\ p_k\rangle \in \mathcal{N}$ and $\seq{p}'=\langle p'_1\ \neg q'_1\ \dots\ \neg q'_{k'-1}\ p'_{k'}\rangle \in \mathcal{N}$ be two NSP s.t. $\seq{p}\nspinclplus\seq{p}'$. Thus, we have that $k=k'$.

Similarly to the proof of Proposition \ref{prop:antimon_lhd}, we can show that any $(\genemb,\totalnonincl)$-embedding of $\seq{p}'$ in $\seq{s}$ induces an $(\genemb,\totalnonincl)$-embedding of $\seq{p}$ in $\seq{s}$.
This enables to conclude that $\weaklycontains^{\totalnonincl}_{\genemb}$ is anti-monotonic on $(\mathcal{N},\nspinclplus)$.

The anti-monotonicity of $\stronglycontains^{\totalnonincl}_{\softemb}$ requires that each embedding of $\seq{p}^+$ in $\seq{s}$ satisfies the negations. 
Let us assume that $p'\stronglycontains^{\totalnonincl}_{\softemb} s$, then there exists an embedding $(e_i)_{i\in[k]}$ of $\seq{p}'$. $(e_i)_{i\in[k]}$ is also an embedding of $\seq{p}'^+$ (Lemma \ref{lemma:pos_embedding}). 
According to 1. in Definition \ref{def:nsp_relations}
and because $k=k'$, $\seq{p}'^+=\seq{p}^+$, and then $(e_i)_{i\in[k]}$ is an embedding of $\seq{p}^+$ in $\seq{s}$. Thus, we have shown that there is at least one embedding of $\seq{p}^+$ in $\seq{s}$.
If $\stronglycontains^{\totalnonincl}_{\softemb}$ is not anti-monotonic, then there exists an embedding $(e_i)_{i\in[k]}$ of $\seq{p}^+$ such that for some $j\in[k]$ and $l\in [e_{j}+1,e_{j+1}-1]$, it is false that $q_j \totalnoninclrel s_l$ ($\exists \alpha\in q_j,\, \alpha \not\in s_l$).
According to 2. in Definition \ref{def:nsp_relations},
$q_j \subseteq q'_j$, and thus it is false $q'_j \totalnoninclrel s_l$.
However, $(e_i)_{i\in[k]}$ is also an embedding of $\seq{p}'^+$. Since $p'\stronglycontains^{\totalnonincl}_{\softemb} s$, it follows that $q'_j \totalnoninclrel s_l$. There is a contradiction, thus $\stronglycontains^{\totalnonincl}_{\softemb}$ is anti-monotonic.

The anti-monotonicity of $\stronglycontains^{\totalnonincl}_{\strictemb}$ requires that each embedding of $\seq{p}^+$ in $\seq{s}$ satisfies the negations. 
Let us assume that $p'\stronglycontains^{\totalnonincl}_{\strictemb} s$, then there exists an embedding $(e_i)_{i\in[k]}$ of $\seq{p}'$. $(e_i)_{i\in[k]}$ is also an embedding of $\seq{p}'^+$ (Lemma \ref{lemma:pos_embedding}). According to 1.\ in Definition \ref{def:nsp_relations}, and because $k=k'$, $\seq{p}'^+=\seq{p}^+$, and then $(e_i)_{i\in[k]}$ is an embedding of $\seq{p}^+$ in $\seq{s}$. Thus, we have shown that there is at least one embedding of $\seq{p}^+$ in $\seq{s}$.
If $\stronglycontains^{\totalnonincl}_{\strictemb}$ is not anti-monotonic, then there exists an embedding $(e_i)_{i\in[k]}$ of $\seq{p}^+$ such that for some $j\in[k]$, it is false that $q_j \totalnoninclrel \bigcup_{l\in [e_{j}+1,e_{j+1}-1]}s_l$.
According to 2.\ in Definition \ref{def:nsp_relations}, $q_j \subseteq q'_j$, and thus it is false $q'_j \totalnoninclrel \bigcup_{l\in [e_{j}+1,e_{j+1}-1]}s_l$.
Nonetheless, $(e_i)_{i\in[k]}$ is also an embedding of $\seq{p}'^+$. And $p'\stronglycontains^{\totalnonincl}_{\strictemb} s$, it implies that $q'_j \totalnoninclrel \bigcup_{l\in [e_{j}+1,e_{j+1}-1]}s_l$. There is a contradiction, thus $\stronglycontains^{\totalnonincl}_{\strictemb}$ is anti-monotonic.
\end{proof}

\begin{proof}[Proof of Proposition \ref{prop:support_dominances}]
Let $\theta, \theta'\in\Theta$, then $\theta \dom \theta' \implies supp_\theta(\seq{p}) \leq supp_{\theta'}(\seq{p})$ for all $\seq{p}\in\mathcal{N}$ (by Definition \ref{def:dominance} of the dominance relation).
Thus, Proposition \ref{prop:support_dominances} comes immediately with Proposition \ref{prop:dominances}.
\end{proof}

\begin{proof}[Proof of Proposition \ref{prop:support_antimonotonie}]
Let $\seq{p}, \seq{p}'\in\mathcal{N}$ be two negative sequential patterns such that $\seq{p} \lhd \seq{p}'$.
According to Proposition \ref{prop:antimon_lhd}, $\seq{p}'\weaklycontains^{\totalnonincl}_{\genemb}\seq{s}\implies \seq{p}\weaklycontains^{\totalnonincl}_{\genemb}\seq{s}$ for all $\seq{s}$. Thus, $supp_{\totalnonincl,\genemb,\weaklycontains}(\seq{p}') \leq supp_{\totalnonincl,\genemb,\weaklycontains}(\seq{p})$.

If $\seq{p} \nspinclplus \seq{p}'$.
According to Proposition \ref{prop:antimon_nspinclplus}, $\seq{p}'\gencontains^{\totalnonincl}_{\genemb}\seq{s}\implies \seq{p}\gencontains^{\totalnonincl}_{\genemb}\seq{s}$ for all $\seq{s}$. Thus, $supp_{\totalnonincl,\genemb,\gencontains}(\seq{p}') \leq supp_{\totalnonincl,\genemb,\gencontains}(\seq{p})$.
\end{proof}

\section{Additional dominance results}

Table \ref{tab:dominances_prop} below illustrates the dominance and non-dominance relations that are given in Proposition \ref{prop:dominances}. In this section, we use transitivity of the dominance relation to complete the table.

\begin{table*}[p]
\caption{Dominance relations from Proposition \ref{prop:dominances}. 
$\dom$ (resp.\ $\not\dom$) in the Table means that the semantic at the left of the row dominates (resp.\ does not dominate) the semantic at the top of the column. The indices are corresponding equation numbers in Proposition \ref{prop:dominances}. }
\label{tab:dominances_prop}

\centering
\begin{tabular}{ccccccccc}
\toprule
 & $(\partialnonincl,\strictemb,\stronglycontains)$ & $(\partialnonincl,\strictemb,\weaklycontains)$ & $(\partialnonincl,\softemb,\stronglycontains)$ & $(\partialnonincl,\softemb,\weaklycontains)$ & $(\totalnonincl,\strictemb,\stronglycontains)$ & $(\totalnonincl,\strictemb,\weaklycontains)$ & $(\totalnonincl,\softemb,\stronglycontains)$ & $(\totalnonincl,\softemb,\weaklycontains)$ \\
 \midrule
 $(\partialnonincl,\strictemb,\stronglycontains)$ & $\cdot$ & $\dom_{(\ref{eq:dominance3})}$ & $\dom_{(\ref{eq:dominance1})}$ & & $\not\dom_{(\ref{eq:nondominance1})}$ & $\not\dom_{(15)}$ & &  \\ 
 $(\partialnonincl,\strictemb,\weaklycontains)$ & $\not\dom_{(\ref{eq:nondominance2})}$ & $\cdot$ & $\not\dom_{(\ref{eq:nondominance4})}$ & $\dom_{(\ref{eq:dominance1})}$ & $\not\dom_{(13)}$ & $\not\dom_{(\ref{eq:nondominance1})}$ & &  \\ 
  $(\partialnonincl,\softemb,\stronglycontains)$ & $\not\dom_{(\ref{eq:nondominance3})}$ & $\not\dom_{(12)}$ & $\cdot$ & $\dom_{(\ref{eq:dominance3})}$ & & & $\not\dom_{(\ref{eq:nondominance1})}$ & \\ 
  $(\partialnonincl,\softemb,\weaklycontains)$ & & $\not\dom_{(\ref{eq:nondominance3})}$ & $\not\dom_{(\ref{eq:nondominance2})}$ & $\cdot$ & & & & $\not\dom_{(\ref{eq:nondominance1})}$\\ 
  $(\totalnonincl,\strictemb,\stronglycontains)$ & $\dom_{(\ref{eq:dominance4})}$ & & & & $\cdot$ & $\dom_{(\ref{eq:dominance3})}$ & $\dom_{(\ref{eq:dominance1})}$ & \\ 
  $(\totalnonincl,\strictemb,\weaklycontains)$ & $\not\dom_{(13)}$ & $\dom_{(\ref{eq:dominance4})}$ & & & $\not\dom_{(\ref{eq:nondominance2})}$ & $\cdot$ & & $\dom_{(\ref{eq:dominance1})}$ \\ 
  $(\totalnonincl,\softemb,\stronglycontains)$ & & & $\dom_{(\ref{eq:dominance4})}$ & & $\dom_{(\ref{eq:dominance2})}$ & & $\cdot$ & $\dom_{(\ref{eq:dominance3})}$ \\ 
  $(\totalnonincl,\softemb,\weaklycontains)$& & & $\not\dom_{(14)}$ & $\dom_{(\ref{eq:dominance4})}$ & & $\dom_{(\ref{eq:dominance2})}$ & $\not\dom_{(\ref{eq:nondominance2})}$ & $\cdot$ \\ 
 
\bottomrule
\end{tabular}
\end{table*}

\subsection{Transitive dominance relation}
Let $\theta\dom\theta'$ and $\theta'\dom\theta''$ then, by transitivity, we have $\theta'\dom\theta''$

\begin{itemize}
\item $\theta=(\totalnonincl,\strictemb,\stronglycontains)$, $\theta'=(\partialnonincl,\softemb,\stronglycontains)$ and $\theta''=(\partialnonincl,\strictemb,\stronglycontains)$. $\theta\dom\theta'$ is obtained by (5) and $\theta'\dom\theta''$ is obtained by (7).
\item ...
\end{itemize}

In Table \ref{tab:dominances_prop}, we deduce that the $x$ cell is $\dom$ when we have the following scheme (or symmetrical schemes): a square with a diagonal of $\dom$ and a diagonal with $x$ and a cell in the diagonal of the matrix. 
\begin{center}
\begin{tabular}{ccc}
$\dom$ & $\dots$ & $x$ \\
$\vdots$ & & $\vdots$ \\
$\cdot$ & $\dots$ & $\dom$ \\
\end{tabular}
\end{center}

Table \ref{tab:dominances_proof_transitivity} illustrates the dominance that can be deduced from the Proposition \ref{prop:dominances}. Table \ref{tab:dominances_proof_transitivity2} contains additional dominance relations deduced by second order transitivity.

% form a square with a dot point such that the two rectangles coin are \dom! See illustration which conclude to the \dom for X !
%
%    \dom -------   X
%     |             |
%     |             |
%     .   -------  \dom 
%

\begin{table*}[p]
\caption{Dominance relations added by transitivity (with indices). 
The indices denote the coordinates (row-columns) which is the diagonally opposite relation in the "transitive-square" (see description in the text).
$\dom$ (resp.\ $\nodom$) in the Table means that the semantics at the left of the row dominates (resp.\ does not dominate) the semantics at the top of the column. 
}
\label{tab:dominances_proof_transitivity}

\centering
\begin{tabular}{ccccccccc}
\toprule
 & $(\partialnonincl,\strictemb,\stronglycontains)$ & $(\partialnonincl,\strictemb,\weaklycontains)$ & $(\partialnonincl,\softemb,\stronglycontains)$ & $(\partialnonincl,\softemb,\weaklycontains)$ & $(\totalnonincl,\strictemb,\stronglycontains)$ & $(\totalnonincl,\strictemb,\weaklycontains)$ & $(\totalnonincl,\softemb,\stronglycontains)$ & $(\totalnonincl,\softemb,\weaklycontains)$ \\
 \midrule
 $(\partialnonincl,\strictemb,\stronglycontains)$ & $\cdot$ & $\dom$ & $\dom$ & $\dom_{3-3}$ & $\nodom$ & $\nodom$ & &  \\ 
 $(\partialnonincl,\strictemb,\weaklycontains)$ & $\nodom$ & $\cdot$ & $\nodom$ & $\dom$ & $\nodom$ & $\nodom$ & &  \\ 
  $(\partialnonincl,\softemb,\stronglycontains)$ & $\nodom$ & $\nodom$ & $\cdot$ & $\dom$ & & & $\nodom$ & \\ 
  $(\partialnonincl,\softemb,\weaklycontains)$ & & $\nodom$ & $\nodom$ & $\cdot$ & & & & $\nodom$\\ 
  $(\totalnonincl,\strictemb,\stronglycontains)$ & $\dom$ & $\dom_{1-1}$ & $\dom_{1-1}$ & & $\cdot$ & $\dom$ & $\dom$ & $\dom_{7-7}$ \\ 
  $(\totalnonincl,\strictemb,\weaklycontains)$ & $\nodom$ & $\dom$ & & $\dom_{2-2}$ & $\nodom$ & $\cdot$ & & $\dom$ \\ 
  $(\totalnonincl,\softemb,\stronglycontains)$ & $\dom_{5-5}$ & & $\dom$ & $\dom_{3-3}$ & $\dom$ & $\dom_{8-8}$ & $\cdot$ & $\dom$ \\ 
  $(\totalnonincl,\softemb,\weaklycontains)$& & $\dom_{6-6}$& $\nodom$ & $\dom$ & & $\dom$ & $\nodom$ & $\cdot$ \\ 
 
\bottomrule
\end{tabular}
\end{table*}

\begin{table*}[p]
\caption{Dominance relations added by second-order transitivity (with indices). 
The indices denote the coordinates (row-columns) which is the diagonally opposite relation in the "transitive-square" (see description in the text).
$\dom$ (resp.\ $\nodom$) in the Table means that the semantics at the left of the row dominates (resp.\ does not dominate) the semantics at the top of the column.}
\label{tab:dominances_proof_transitivity2}

\centering
\begin{tabular}{ccccccccc}
\toprule
 & $(\partialnonincl,\strictemb,\stronglycontains)$ & $(\partialnonincl,\strictemb,\weaklycontains)$ & $(\partialnonincl,\softemb,\stronglycontains)$ & $(\partialnonincl,\softemb,\weaklycontains)$ & $(\totalnonincl,\strictemb,\stronglycontains)$ & $(\totalnonincl,\strictemb,\weaklycontains)$ & $(\totalnonincl,\softemb,\stronglycontains)$ & $(\totalnonincl,\softemb,\weaklycontains)$ \\
 \midrule
 $(\partialnonincl,\strictemb,\stronglycontains)$ & $\cdot$ & $\dom$ & $\dom$ & $\dom$ & $\nodom$ & $\nodom$ & &  \\ 
 $(\partialnonincl,\strictemb,\weaklycontains)$ & $\nodom$ & $\cdot$ & $\nodom$ & $\dom$ & $\nodom$ & $\nodom$ & &  \\ 
  $(\partialnonincl,\softemb,\stronglycontains)$ & $\nodom$ & $\nodom$ & $\cdot$ & $\dom$ & & & $\nodom$ & \\ 
  $(\partialnonincl,\softemb,\weaklycontains)$ & & $\nodom$ & $\nodom$ & $\cdot$ & & & & $\nodom$\\ 
  $(\totalnonincl,\strictemb,\stronglycontains)$ & $\dom$ & $\dom$ & $\dom$ & $\dom_{6-6}$ & $\cdot$ & $\dom$ & $\dom$ & $\dom$ \\ 
  $(\totalnonincl,\strictemb,\weaklycontains)$ & $\nodom$& $\dom$ & & $\dom$ & $\nodom$ & $\cdot$ & & $\dom$ \\ 
  $(\totalnonincl,\softemb,\stronglycontains)$ & $\dom$ & $\dom_{5-5}$ & $\dom$ & $\dom$ & $\dom$ & $\dom$ & $\cdot$ & $\dom$ \\ 
  $(\totalnonincl,\softemb,\weaklycontains)$& & $\dom$& $\nodom$ & $\dom$ & & $\dom$ & $\nodom$ & $\cdot$ \\ 
 
\bottomrule
\end{tabular}
\end{table*}

\subsection{Transitive non-dominance relation}
\begin{lemma}
Let $\theta$, $\theta'$ and $\theta''$ s.t. $\theta\dom\theta'$ and $\theta\not\dom\theta''$ then $\theta'\not\dom\theta''$%.
\end{lemma}
\begin{proof}
By absurd, suppose $\theta'\dom\theta''$ thus, by transitivity, $\theta\dom\theta''$ that is not possible.
\end{proof}

In Table \ref{tab:dominances_prop}, we deduce that the $x$ cell is $\not\dom$ when we have the following scheme (or symmetrical schemes): a square with a diagonal of $\dom$ and $x$; and a diagonal with $\not\dom$ and a cell in the diagonal of the matrix. 
\begin{center}
\begin{tabular}{ccc}
$\dom$ & $\dots$ & $\not\dom$ \\
$\vdots$ & & $\vdots$ \\
$\cdot$ & $\dots$ & $x$ \\
\end{tabular}
\end{center}

Table \ref{tab:dominances_proof_transitivity3} illustrates the non dominance relations that can be deduced from previously deduced dominance and non-dominance.
% form a square with a dot point such that the two rectangles coin are \dom! See illustration which conclude to \not\dom for X
%
%    \dom ------- \not\dom
%     |             |
%     |             |
%     .   -------   X 
% or, symmetrically
%     X ------- \not\dom
%     |             |
%     |             |
%     .   -------  \dom 

\begin{table*}[p]
\caption{Non-dominance relations added by transitivity. 
The indices denote the coordinates (row-columns) which is the diagonally opposite relation in the "transitive-square" (see description in the text).
$\dom$ (resp.\ $\nodom$) in the Table means that the semantics at the left of the row dominates (resp.\ does not dominate) the semantics at the top of the column.
}
\label{tab:dominances_proof_transitivity3}

\centering
\begin{tabular}{ccccccccc}
\toprule
 & $(\partialnonincl,\strictemb,\stronglycontains)$ & $(\partialnonincl,\strictemb,\weaklycontains)$ & $(\partialnonincl,\softemb,\stronglycontains)$ & $(\partialnonincl,\softemb,\weaklycontains)$ & $(\totalnonincl,\strictemb,\stronglycontains)$ & $(\totalnonincl,\strictemb,\weaklycontains)$ & $(\totalnonincl,\softemb,\stronglycontains)$ & $(\totalnonincl,\softemb,\weaklycontains)$ \\
 \midrule
 $(\partialnonincl,\strictemb,\stronglycontains)$ & $\cdot$ & $\dom$ & $\dom$ & $\dom$ & $\nodom$ & $\nodom$ & $\nodom_{7-5}$ & $\nodom_{8-6}$ \\ 
 $(\partialnonincl,\strictemb,\weaklycontains)$ & $\nodom$ & $\cdot$ & $\nodom$ & $\dom$ & $\nodom$ & $\nodom$ & $\nodom_{7-6}$ & $\nodom_{8-6}$ \\ 
  $(\partialnonincl,\softemb,\stronglycontains)$ & $\nodom$ & $\nodom$ & $\cdot$ & $\dom$ & $\nodom_{1-3}$& $\nodom_{6-2}$ & $\nodom$ & $\nodom_{8-2}$\\ 
  $(\partialnonincl,\softemb,\weaklycontains)$ & $\nodom_{1-2}$& $\nodom$ & $\nodom$ & $\cdot$ & $\nodom_{1-4}$& $\nodom_{6-8}$ & $\nodom_{8-4}$& $\nodom$\\ 
  $(\totalnonincl,\strictemb,\stronglycontains)$ & $\dom$ & $\dom$ & $\dom$ & $\dom$ & $\cdot$ & $\dom$ & $\dom$ & $\dom$ \\ 
  $(\totalnonincl,\strictemb,\weaklycontains)$ & $\nodom$ & $\dom$ & $\nodom_{8-6}$ & $\dom$ & $\nodom$ & $\cdot$ & $\nodom_{8-6}$ & $\dom$ \\ 
  $(\totalnonincl,\softemb,\stronglycontains)$ & $\dom$ & $\dom$ & $\dom$ & $\dom$ & $\dom$ & $\dom$ & $\cdot$ & $\dom$ \\ 
  $(\totalnonincl,\softemb,\weaklycontains)$& $\nodom_{6-8}$ & $\dom$& $\nodom$ & $\dom$ & $\nodom_{6-8}$ & $\dom$ & $\nodom$ & $\cdot$ \\ 
 
\bottomrule
\end{tabular}
\end{table*}
\end{document}